\documentclass[twoside]{article}

\usepackage[accepted]{aistats2022}
\usepackage{natbib}
\usepackage{booktabs}
\usepackage{graphicx}
\usepackage{amsmath}
\usepackage{amssymb}
\usepackage{amsfonts}
\usepackage{bm}
\usepackage{mathtools}
\usepackage{multirow}
\usepackage{relsize}
\usepackage{wrapfig,framed,caption}
\usepackage{soul}
\usepackage{xfrac}
\usepackage{amsthm}
\usepackage{thmtools,thm-restate}
\usepackage{hyperref}       
\usepackage{url}            
\usepackage{booktabs}       
\usepackage{amsfonts}       
\usepackage{nicefrac}       
\usepackage{microtype}      
\usepackage{xcolor}         
\usepackage{microtype}
\usepackage{graphicx}
\usepackage[scriptsize,hang]{subfigure}
\usepackage{booktabs} 
\usepackage{pgfplots}
\usepackage{enumitem}
\usepackage{relsize}
\usepackage{algorithm}
\usepackage[noend]{algpseudocode}
\usepackage[group-separator={,}]{siunitx}
\usepackage{csvsimple}
\usepackage{siunitx}
\usepackage[normalem]{ulem}
\usepackage{siunitx}
\newcommand{\srange}[2] {\lbrack\num{#1}, \num{#2}\rbrack}

\newcommand{\Delt}{\bm{\Delta}}
\DeclareMathOperator{\vect}{vec}

\DeclareMathOperator{\sign}{sign}
\DeclareMathOperator*{\argmin}{arg\,min}

\DeclareMathOperator{\softmax}{softmax}
\DeclareMathOperator{\dom}{dom}

\definecolor{darkgreen}{rgb}{0.09, 0.45, 0.27}
\definecolor{mygold}{rgb}{1,0.75,0}
\definecolor{myred}{rgb}{0.75,0,0}
\definecolor{mycharcoal}{RGB}{58,55,51}
\definecolor{myblue}{RGB}{52,57,176}

\newcommand{\new}{\text{\tiny{new}}}

\newcommand{\ah}{\hat{\a}}

\newcommand{\soft}{\text{\tiny SM}}

\newcommand{\R}{\mathbb{R}}

\newcommand{\one}{\bm{1}}

\renewcommand{\a}{\bm{a}}

\newcommand{\x}{\bm{x}}
\newcommand{\w}{\bm{w}}
\newcommand{\y}{\bm{y}}

\newcommand{\I}{\bm{I}}

\newcommand{\z}{\bm{z}}

\renewcommand{\H}{\bm{H}}
\newcommand{\yh}{\hat{\y}}

\newcommand{\dint}{\mathrm{d}}

\newcommand{\W}{\bm{W}}

\newcommand{\matching}{f}

\definecolor{niceblue_col}{RGB}{56,108,176}
\definecolor{nicered_col}{RGB}{215,48,39}
\definecolor{nicedarkred_col}{RGB}{215,48,39}

\newcommand{\Wt}{\widetilde{\W}}

\begin{document}

%
\runningtitle{LocoProp: Enhancing BackProp via Local Loss Optimization}

%
\runningauthor{Ehsan Amid, Rohan Anil, Manfred K. Warmuth}

\twocolumn[

\aistatstitle{LocoProp: Enhancing BackProp via Local Loss Optimization}
\aistatsauthor{ Ehsan Amid* \And Rohan Anil* \And Manfred K. Warmuth}
\aistatsaddress{ Google Research, Brain Team \And Google Research, Brain Team  \And Google Research, NY} ]
\begin{abstract}
  Second-order methods have shown state-of-the-art performance for optimizing deep neural networks. Nonetheless, their large memory requirement and high computational complexity, compared to first-order methods, hinder their versatility in a typical low-budget setup. This paper introduces a general framework of layerwise loss construction for multilayer neural networks that achieves a performance closer to second-order methods while utilizing first-order optimizers only. Our methodology lies upon a three-component loss, target, and regularizer combination, for which altering each component results in a new update rule. We provide examples using squared loss and layerwise Bregman divergences induced by the convex integral functions of various transfer functions. Our experiments on benchmark models and datasets validate the efficacy of our new approach, reducing the gap between first-order and second-order optimizers.
\end{abstract}

\section{INTRODUCTION}

Backpropagation (or \emph{BackProp} for short)~\citep{backprop} has been the prominent technique for training neural networks. BackProp is simply an expansion of the chain rule for calculating the derivative of the output \emph{loss function} with respect to the weights in each layer. The BackProp update involves a forward pass to calculate the network's activations given the input batch of training data. After the forward pass, the gradients of the loss function with respect to the network weights are \emph{backpropagated} from the output layer all the way to the input and the weights are updated by applying a single gradient step.

Stochastic gradient descent is the most basic update rule, which involves taking a step in the direction of the negative backpropagated gradients. The more advanced first-order techniques such as \mbox{AdaGrad}~\citep{adagrad}, RMSprop~\citep{Tieleman2012}, and Adam~\citep{kingma2014adam} involve preconditioning the gradient by a diagonal matrix, as well as incorporating momentum. In contrast, second-order optimizers such as \mbox{Shampoo}~\citep{gupta2018,anilpractical} and K-FAC~\citep{Heskes2000OnNL,martens2015optimizing,ba2016distributed} 
use Kronecker products to approximately form a full-matrix preconditioner, i.e., Full Matrix AdaGrad and Natural Gradient~\citep{amari}, respectively. Shampoo and K-FAC achieve significantly faster convergence than the first-order methods in terms of the number of steps and wall-time. However, large memory requirements ~\citep{anil2019,shazeer2018} and the high computational cost of calculating matrix inverses (inverse-$p$th roots, in the case of Shampoo) make these approaches prohibitive for larger models. These methods require further extensions such as block diagonalization to make them scalable, and demand high precision arithmetic for computing the inverses. Moreover, they require more resources and are difficult to  parallelize~\citep{anilpractical,ba2016distributed,osawa18}. The question remains whether it is possible to achieve a similar performance to second-order methods with first-order optimizers without any extra forward-backward passes per batch or explicitly forming the preconditioners.

In this paper, we propose a simple local loss construction and minimization approach to enhance the efficacy of every BackProp step that is easily scalable to modern large-scale architectures. Our method involves forming a local loss for each layer by fixing a target and then minimizing this loss iteratively. The updates for each layer are entirely decoupled and run in parallel. A single iteration on the local loss always recovers a single BackProp step. Thus, any further iterations on the local loss work towards enhancing the initial BackProp step. Our approach is comparable to first-order methods in terms of memory and computation requirements: It involves a single forward-backward pass on each batch, followed by a few parallelizable local iterations. Our approach significantly improves the convergence of first-order methods, thus reducing the gap between first-order and second-order optimization techniques.

\subsection{Related Work}
Alternative update rules or extensions to BackProp have been proposed over the years~\citep{perpinan,carreira2016parmac,admm-training,local-hessian,DecouplingBU,contrastive,cd-global,askari2018lifted,lifted_proximal,gu2020fenchel}. For instance, the \emph{Difference Target Propagation} method~\citep{targetprop} involves estimating a target value in each layer to approximate the BackProp gradient. The targets are formed via a separate network which performs as the inverse function. The main goal of such approaches is decoupled training~\citep{pmlr-v70-jaderberg17a} by approximating the gradients or introducing auxiliary variables. Such decoupled training approaches are not the focus of our current work. 

In this paper, we mainly focus on improving the performance of first-order optimizers~\citep{nesterov,adagrad,Tieleman2012,kingma2014adam} to match the performance of their second-order contenders, namely K-FAC~\citep{Heskes2000OnNL,martens2015optimizing} and Shampoo~\citep{gupta2018,anilpractical}, via a local loss construction approach. The local problems are solved approximately using first-order optimizers, while we show that the exact solutions to these problems recover second-order update rules such as K-FAC. Our construction also recovers Proximal Backpropagation~\citep{proximal}, which builds upon the idea of local squared loss minimization, as a special case (analysis relegated to the appendix). Closely related is the work of~\cite{johnson2020guided} which proposes finding 
a guide function that is ahead of the current model with respect to the loss minimization and pushing the model towards the guide. They show that this construction improves training efficiency and is related to self-distillation.

\subsection{Notation}
We adopt the following notation throughout the paper. We use $(\x, \y)$ to denote the input instance and target label pair. For an $M$-layer neural network and at a given layer $m \in [M]$, we use $\ah_m$ (respectively, $\yh_m$) and $\a_m$ (respectively, $\y_m$) 
for the predicted and target pre (post)-activations, respectively.
Note that $\ah_m = \W_m\,\yh_{m - 1}$ and $\yh_m = f_m(\ah_m)$ where $\W_m$ is the weight matrix\footnote{We assume that biases are incorporated in the weights.} and $f_m$ is an elementwise 
non-decreasing transfer function. (Also using this
notation, we have $\yh_0 = \x$ and $\yh = \yh_M$, but we distinguish between
$\y_M$ and $\y$.)
We denote the loss of the network in the
final layer (which is used for training with standard BackProp) as $L(\y, \yh) = L(\y, \yh_M)$. We make no assumption on the final loss $L$ other than differentiability.

Lastly, $\odot$ and $\otimes$ denote Hadamard product and Kronecker product, respectively, and $\mathbb{I}$ denotes the indicator function.

\subsection{Our Contributions}
\begin{itemize}[leftmargin=4mm]
\vspace{-0.1cm}
\setlength\itemsep{0em}
\item We introduce a layerwise loss construction framework that relies on a three-component loss, target, and regularizer combination. Our construction allows flexibility in setting each component to recover multiple different update rules.
\item We discuss the motivation behind several combinations and show the connection to implicit gradient updates, Proximal Backpropagation, and second-order methods such as K-FAC. In each case, we show that the combination is set such that a single iteration on the local objective recovers BackProp (or a more advanced update such as natural gradient descent~\citep{amari}), while applying further iterations recovers a second-order update.
\item Experimentally, we show that performing additional iterations significantly improves convergence of the first-order methods on a benchmark optimization problem, thus reducing the gap between first-order and second-order techniques. After computing the BackProp gradients, our method is embarrassingly parallelizable across layers~\citep{nielsen2016introduction}.
\end{itemize}
\begin{figure}[t!]
    \centering
    \includegraphics[width=0.35\textwidth]{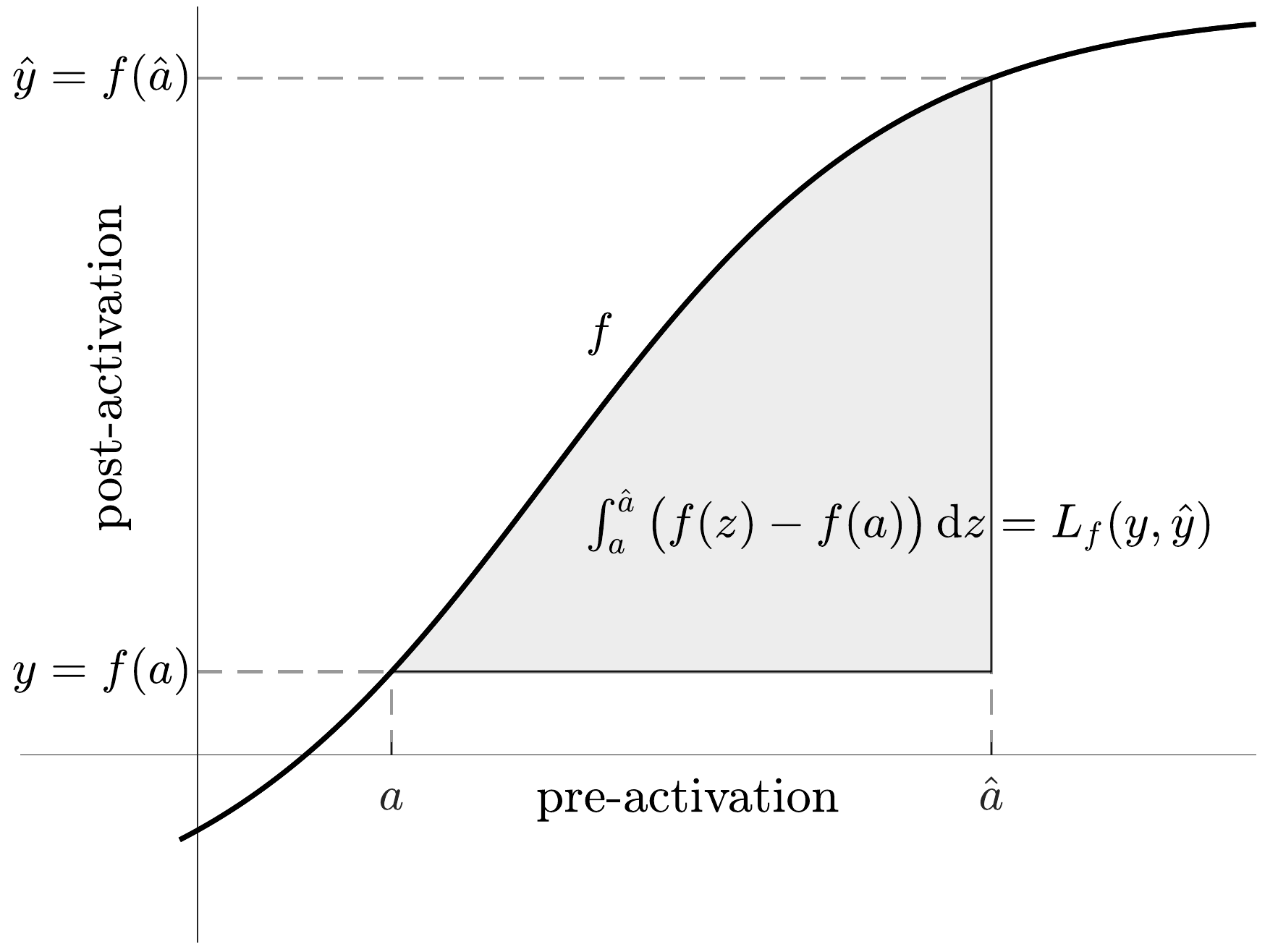}
\vspace{-0.2cm}
\caption{The matching loss $L_f$ of an 
    increasing transfer function $f$ as an integral under the curve $f$
    which is always convex in $\hat{a}$.}\label{fig:matching}
    \vspace{-0.1cm}
\end{figure}
In particular, we will discuss in detail two variants of our local loss construction approach. The first method relies on using the squared loss as the local loss. Except for setting the targets, the squared loss construction approach uses no information about the local transfer function. Instead, the second variant, based on the \emph{matching loss}, adapts the local loss function to the transfer function of the layer.

Matching loss was initially introduced as a line integral~\citep{match} and was later motivated 
via a convex duality argument involving Bregman divergences~\citep{multidim,bitemp}. 
Specifically, consider an
elementwise strictly increasing transfer function $f:\,
\R^d \rightarrow \R^d$. Given input $\x \in \R^n$ and
target label $\y \in \R^d$, let $\ah \in \R^d$ be the predicted
pre (transfer function) activation produced by the model 
(for instance, the linear activation before softmax in a single-layer logistic regression classifier). 
The (post-activation) prediction of the model is given by $\yh = f(\ah)$
(in this example, the softmax probabilities). 
The goal of training is to make the model prediction $\yh$
closer to the target label $\y$
by means of minimizing a measure of
discrepancy, commonly known as a loss function. In particular, the
\emph{matching loss} of the transfer function $f$ between
the target $\y$ and the prediction $\yh = \yh(\ah)$ is
defined as the following line integral of $f$,
\begin{equation}
    \label{eq:matching}
    L_{\matching}\big(\y, \yh(\ah)) \doteq \int_{\a}^{\ah} \big(f(\z) - f(\a)\big)^\top \dint \z\,,
\end{equation}
where $\a = f^{-1}(\y)$ is the target pre-activation. 
Figure~\ref{fig:matching} depicts this integral when $d=1$.
By definition, the gradient of the matching loss with respect to $\ah$ 
admits a simple form in terms of the difference between the prediction and the target:
\begin{equation}
\label{eq:delta}
\nabla_{\ah} L_{\matching}\big(\y, \yh(\ah)) = f(\ah) - f(\a) = \yh - \y\, .
\end{equation}
Consequently, when $f$ is elementwise strictly increasing, then the local matching loss is strictly convex with respect to $\ah$. 
Note that the majority of the non-linear transfer functions used
in practice (e.g., leaky ReLU, $\tanh$, softmax, etc.) are indeed
(elementwise) strictly increasing and thus their matching losses are strictly convex.%
\footnote{The integral of the softmax function, known as ``log-sum-exp'', is strictly convex in $\R^d - \{\pm c\one_d,\, c \in \R_+\}$.} 

As we shall see, the logistic loss is one of the main examples of a matching loss.
When the logistic loss (i.e., softmax transfer function and KL (or cross entropy) divergence) is used for multilayer neural networks,
then the overall loss is strictly convex in the weights of the last layer.
Nonetheless, the convexity does not necessarily extend to the weights of the layers below. These weights are then updated by backpropagating the gradient of the \emph{same} final loss. As one of our main contributions, we extend this construction to layerwise convex losses.

\section{LOCAL LOSS OPTIMIZATION}
Given an input example\footnote{The construction generalizes to a batch of examples trivially.} $\x$, our local loss construction framework splits an $M$-layer feedforward neural network into independent layerwise (i.e., single neuron) problems. In layer $m\in[M]$, we fix the input to the layer $\yh_{m-1}$ and define a target $\a_m$ (or $\y_m = f_m(\a_m)$), which lies in the pre (or post) activation domain. The update proceeds by minimizing a loss between the current prediction of the layer and a target $\a_m$, plus a regularizer term on the weights,
\begin{equation*}
    \W^{\new} = \argmin_{\Wt}\big\{\underbrace{\!D(\Wt\yh_{m-1}, \a_m)}_{\text{loss}} + \underbrace{R(\Wt, \W)}_{\text{regularizer}}\big\},
\end{equation*}
\vspace{-0.5cm}

where $D$ denotes a divergence function and $R$ is a regularizer (e.g., a squared $\mathrm{L}_2$-norm). This formulation allows flexibility in term of setting the loss (i.e., divergence function), the target, and the regularizer term. In the next section, we will show how changing each component results in a different update rule, one of which recovers the K-FAC method. In each case, we pick the combination such that a single fixed-point iteration on the objective recovers standard gradient descent (i.e., BackProp) or a more advanced update such as natural gradient descent~\citep{amari}. 
We will refer to our \textbf{Loc}al Loss \textbf{O}ptimization framework as \textbf{LocoProp}.

\subsection{Local Squared Loss}
\label{sec:locos}
We now motivate the more basic approach for constructing the local problems using a squared loss. Given the current pre-activation $\ah_m = \W_m\,\yh_{m-1}$ at layer $m \in [M]$, we first define the \emph{gradient descent (GD) target} of the layer as $\a_m = \ah_m - \gamma\, \nabla_{\ah_m} L(\y, \yh)$ where $\gamma > 0$ is an activation step size. That is, our target corresponds to a GD step on the current pre-activation using the gradient of the final loss with respect to the pre-activation. Keeping the input to the layer $\yh_{m-1}$ fixed, our local optimization problem at layer $m \in [M]$ consists of minimizing the squared loss between the new pre-activation $\Wt\yh_{m-1}$ and the GD target $\a_m$, plus a squared $\mathrm{L}_2$-norm regularizer which keeps the updated weights close to the current weights,
\begin{equation}
    \label{eq:squared_obj}
    \sfrac{1}{2}\, \Vert\Wt\yh_{m-1} - \a_m\Vert^2 + \sfrac{1}{2\eta}\, \Vert \Wt - \W_m\Vert^2\, .
\end{equation}
\vspace{-0.6cm}

Here, $\eta > 0$ controls the trade-off between minimizing the loss and the regularizer. Setting the derivative of the objective to zero, we can write $\W_m^{\new}$ as the solution of a fixed point iteration,
\begin{equation}
\label{eq:gd_fixed}
    \W^{\new}_m = \W_m - \eta\, \big(\W^{\new}_m\yh_{m-1} - \a_m\big)\,\yh_{m-1}^\top\, .
\end{equation}
Interestingly, one iteration over Eq. \eqref{eq:gd_fixed} by replacing $\W_m^{\new}$ with $\W_m$ on the r.h.s. yields,
\begin{align}
\W^{\new}_m & \approx \W_m - \eta\, \big(\W_m\yh_{m-1} - \a_m\big)\,\yh_{m-1}^\top\nonumber\\
& = \W_m - \eta\, \big(\W_m\yh_{m-1}\nonumber\\
& \qquad\quad\quad - (\W_m\yh_{m-1} -\gamma\, \nabla_{\ah_m} L(\y, \yh))\big)\,\yh_{m-1}^\top\nonumber\\
& = \W_{m} - \eta\, \gamma\,
    \nabla_{\ah_m} L(\y, \yh)\,\,\yh_{m-1}^\top\nonumber\\
    & = \W_{m} - \eta_e\,
    \frac{\partial L(\y, \yh)}{\partial
       \ah_m}\frac{\partial \ah_m}{\partial \W_m}\nonumber\tag{BackProp}\, . 
\end{align}
Thus, a single iteration on the objective in Eq~\eqref{eq:squared_obj} recovers BackProp, with an effective learning rate of $\eta_e \doteq \eta\, \gamma > 0$. Nonetheless, Eq.~\eqref{eq:gd_fixed} can in fact be solved in a closed-form as,
\begin{equation*}
    \W_m^{\new} = \W_m - \eta_e\,\nabla_{\W_m} L(\y, \yh)\,\big(\I + \eta\,\yh_{m-1}\yh_{m-1}^\top\big)^{-1},
\end{equation*}
which corresponds to an \emph{implicit gradient} update~\citep{hassibi,pnorm}. Implicit updates can be motivated as the backward Euler approximation of the continuous-time gradient flow~\citep{reparam} and generally provide faster convergence compared to their approximate (explicit) counterpart~\citep{amid2020implicit}. The matrix preconditioner which trails the BackProp gradient also corresponds to the right preconditioner that appears in the K-FAC update rule (see the appendix). In fact, the following proposition states that we can also recover the full K-FAC update as a special case.
\begin{restatable}{proposition}{loco-s-kfac}
\label{prop:loco-s-kfac}
The update that minimizes Eq.~\eqref{eq:squared_obj} with a natural gradient descent target $\a_m = \ah_m - \gamma\, \bm{F}^{-1}_m \nabla_{\ah_m} L(\y, \yh)$, in which $\bm{F}_m$ is the Fisher Information matrix treating the pre-activations as parameters, recovers the K-FAC update.
\end{restatable}
Instead of forming the matrix inverse, which is computationally expensive for large layers, we apply the fixed-point update in Eq.~\eqref{eq:gd_fixed} for a certain number of iterations that suits the computational budgets. We refer to our local squared loss construction approach with GD targets as \emph{LocoProp-S} for short.

\subsection{Matching Loss}
We now extend the local loss construction in the previous section by replacing the squared loss between the pre-activations and the targets with a Bregman divergence which is tailored to the transfer function of each layer. We start by reviewing the idea of a matching loss of a transfer function and show that when using a matching loss, the local problem remains convex with respect to the weights in each layer. Additional LocoProp variants based on dual of the matching loss as well as those obtained by using different weight regularizers are given in the appendix.
\subsubsection{Bregman Divergence}

The integral in Eq.~\eqref{eq:matching} expands
to a \emph{Bregman divergence}~\citep{bregman}
induced by the strictly convex
integral function $F:\, \R^d \rightarrow \R$ such that $f = \nabla F$:
\begin{equation}
\label{eq:bregman}
\begin{split}
L_{\matching}&\big(\y, \yh(\ah)\big) = \Big(F(\z) - f(\a)^\top \z\Big)\Big\vert_{\a}^{\ah}\\
& = F(\ah) - F(\a) - f(\a)^\top (\ah - \a)
\doteq D_F(\ah, \a)\, .
\end{split}
\end{equation}
Bregman divergences are non-negative distance measures that satisfy many desirable properties including:
\textbf{(I) Convexity:} $D_F(\ah, \a)$ is always convex
in the first argument, but not necessarily in the second argument.
\textbf{(II) Duality:} $D_F(\ah, \a) = D_{F^*}(\y,
    \yh)$ where $F^*(\y) \doteq \sup_{\z} \{\y\cdot\z -
	F(\z)\}$ is the Fenchel dual~\citep{urruty} of the
	convex function $F$, and $(\ah, \yh)$ and $(\a,
	\y)$ are pairs of dual points,
    $\yh = f(\ah)\,,\, \ah = f^*(\yh)\,,\,
    \y = f(\a)\,,\, \a = f^*(\y)\, ,$
where $f^* = \nabla F^* = f^{-1}$.
\textbf{(III) Strict non-negativity:} $D_F(\ah, \a) \geq 0$ for all $\ah, \a \in \dom(F)$ and $D_F(\ah, \a) = 0$ iff $\ah = \a$.

Since Bregman divergence are convex with respect to the first argument (Property (I)),
the rewrite \eqref{eq:bregman} of the matching loss as a
Bregman divergence implies that this loss is convex in the pre-activation $\ah$.
Additionally, using the duality argument (Property (II)), we
have $L_{\matching}\big(\y, \yh\big) = D_{F^*}(\y,
\yh)$. Consequently, Property (III) ensures that
$L_{\matching}\big(\y, \yh\big) \geq 0$ and when $f$ is
(elementwise) strictly increasing, $L_{\matching}\big(\y, \yh\big) = 0$ iff $\y = \yh$ (or $\ah = \a$).

A classical example of a matching loss is the commonly used
logistic loss for classification. The logistic loss amounts
to the relative entropy divergence (a.k.a. KL divergence)%
\footnote{The last two terms cancel when $\y$ and $\yh$ are
probability distributions.}:
\begin{equation}
\label{eq:kl}
\text{KL}(\y, \yh) \doteq \sum_i y_i\log\frac{y_i}{\hat{y}_i} - y_i + \hat{y}_i\, ,
\end{equation}
between the target $\y$ and the softmax probabilities, $\yh = f_{\soft}(\ah) = \softmax(\ah) \doteq \frac{\exp(\ah)}{\sum_i \exp \hat{a}_i}$.
The integral of the softmax function is the so called ``log-sum-exp'' function, $F_{\soft}(\a) = \log\sum_i\exp(a_i)$. KL divergence is a Bregman divergence induced by the negative Shannon entropy function $F_{\soft}^*(\y) = \sum_i (y_i \log y_i - y_i)$, which is in fact the Fenchel dual of $F_{\soft}$ (restricted to the unit simplex~\citep{thesis}). Thus, by Property (II), the KL divergence in Eq.~\eqref{eq:kl} between the label $\y$ and the output $\yh$ is equal to the Bregman divergence induced by the convex function $F_{\soft}$ formed between the pre-activation of the final layer $\ah$ and the target $\a = f_{\soft}^{-1}(\y) = \log \y - \frac{1}{d} \sum_i \log y_i\,\one$,
\[
D_{F_{\soft}}(\ah, \a) = \log\frac{\sum_i \exp(\hat{a}_i)}{\sum_i \exp(a_i)} - \sum_i \frac{\exp(a_i)\, (\hat{a}_i - a_i)}{\sum_j \exp(a_j)}\, .
\]
Consequently, neural networks trained with logistic loss (i.e., softmax transfer function paired with KL divergence) aim to minimize the matching loss of the softmax transfer function\footnote{Although the target pre-activation $\a$ may be unbounded in the case of a one-hot $\y$ vector, the construction still carries out.}, \[
L_{f_{\soft}}(\y, \yh) = \text{KL}(\y, \yh) = D_{F_{\soft}}(\ah, \a)\, .
\]
Also, the matching loss ensures convexity with respect to to the weights
of the last layer since $\ah$ is linear in the weights, i.e., $\ah = \ah_M = \W\yh_{M-1}$ where $\yh_{M-1}$ is the input to the last layer.

\begin{figure}[t!]
\begin{center}
    \subfigure[]{\includegraphics[height=0.465\linewidth]{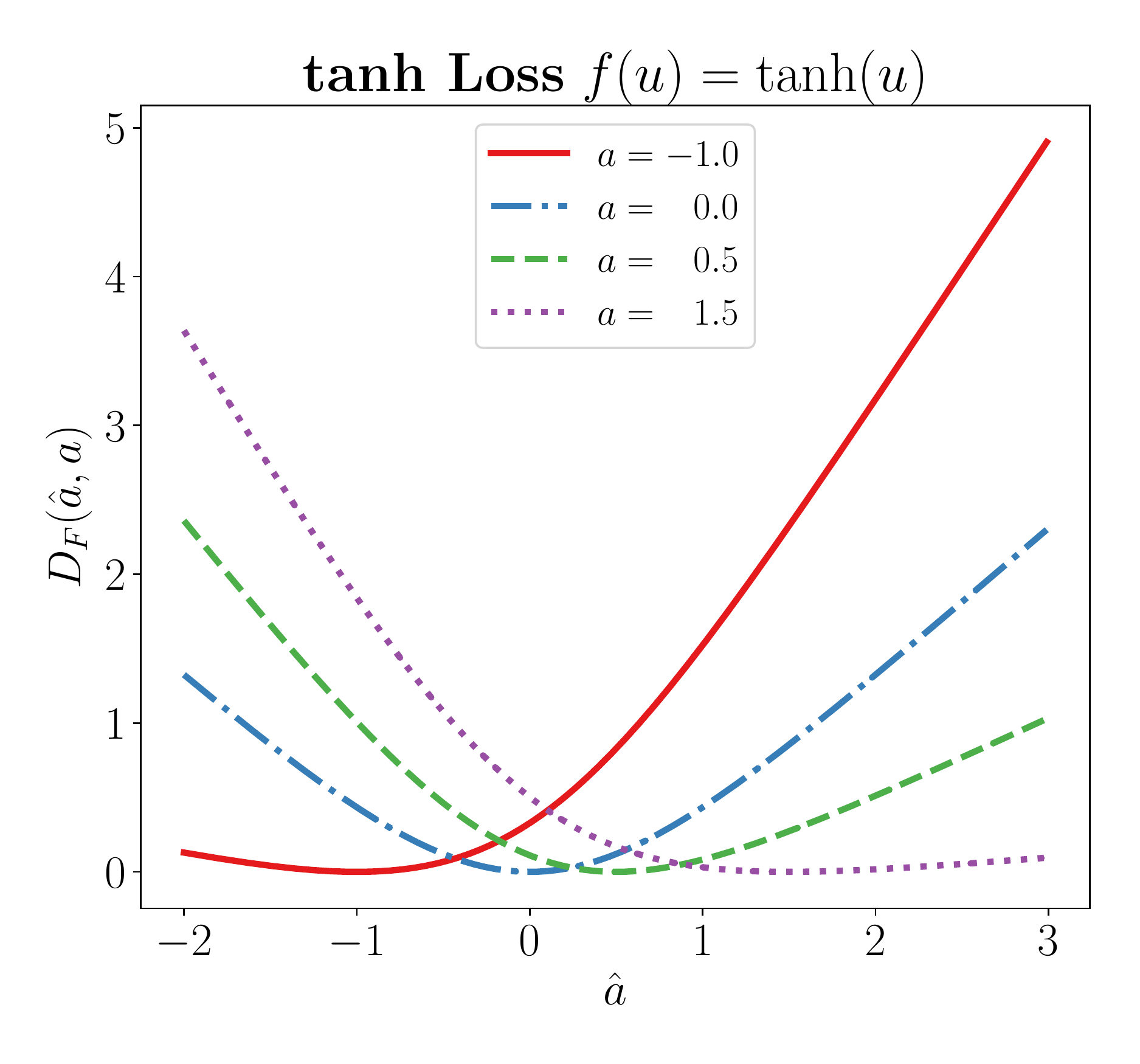}}  \subfigure[]{\includegraphics[height=0.465\linewidth]{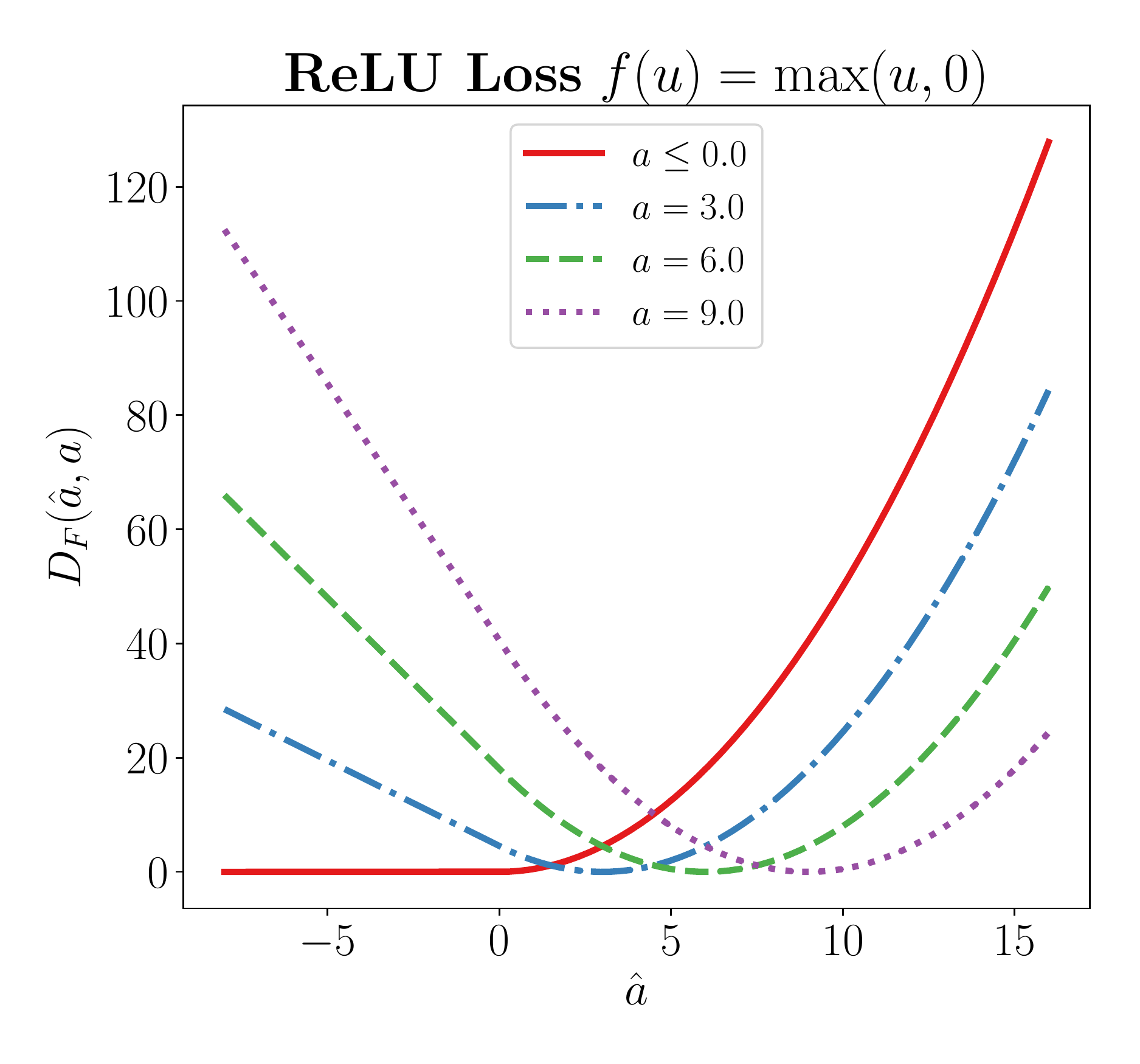}}
    \vspace{-0.3cm}
    \caption{Matching losses induced by the (a) $\tanh$ and (b) ReLU transfer functions. For strictly increasing transfer functions such as $\tanh$, the matching loss is equal to zero iff $\ah = \a$. On the contrary, non-decreasing (but not strictly increasing) transfer functions such as ReLU only satisfy the if condition, thus inducing matching losses with flat regions. (For instance for ReLU, $D_F(\hat{a}, a) = 0$ for all $\hat{a}, a \leq 0$\,.)}
    \label{fig:two_losses}
    \vspace{-0.5cm}
    \end{center}
\end{figure}

\vspace{-0.2cm}
\subsubsection{Common Transfer Functions}
Interestingly, the majority of the commonly used transfer
functions in the modern neural network architectures are
elementwise non-decreasing. Examples of such transfer
functions include (leaky) ReLU, hyperbolic tan, softplus,
softmax (sigmoid), etc. (see the appendix for an extensive list). To simplify the construction, we
refer to the Bregman divergence $D_F(\ah, \a)$ induced by
the convex integral function $F$ as the matching loss of
the transfer function $f$. Although in some cases (such as
softplus) the matching loss has a sophisticated form, in
practice, the only requirement is to calculate the gradient
with respect to the first argument, i.e., $\nabla_{\ah} D_F(\ah,
\a)$. Favorably, Bregman divergences facilitate this
calculation by providing a general form for the derivative
with respect to the first argument; regardless of the form, the
derivative only requires evaluation of the transfer function, as in Eq.~\eqref{eq:delta}.

Lastly, transfer functions such as step function and ReLU
are non-decreasing but not strictly increasing. Thus, the
integral function does not induce a strictly convex
function. As a result, $D_F(\ah, \a) = 0$ does not necessarily imply
$\ah = \a$. Also, the inverses $\a = f^{-1}(\y)$ are not uniquely defined in this case. Figure~\ref{fig:two_losses} illustrates an example. Nonetheless, our construction still applies for
non-decreasing transfer functions as we discuss in the next section. Thus, we will loosely refer  to the divergences induced by non-decreasing transfer functions as Bregman divergences.


\begin{figure*}[t]
\vspace{-0.7cm}
\centering
\scalebox{0.78}{
    \begin{minipage}{0.63\linewidth}
\begin{algorithm}[H]
\caption{LocoProp-S: LocoProp Using Squared Loss}\label{alg:locoprop-s}
\begin{algorithmic}
    \State \textbf{Input} weights $\{\W_m\}$ where $m \in [M]$ for an $M$-layer network, activation step size $\gamma$, weight learning rate $\eta$
\Repeat
    \State $\bullet$\, perform a \textbf{forward pass} and fix the \emph{inputs} $\{\yh_{m-1}\}$
    \State $\bullet$\, perform a \textbf{backward pass} and set the \emph{GD targets}
    \vspace{-0.1cm}\[\a_{m} = \ah_m - \gamma\, \nabla_{\ah_m} L(\y, \yh)\]\vspace{-0.5cm}
    \For{each layer $m \in [M]$ \textbf{in parallel}}
    \vspace{0.05cm}
    \For{$T$ iterations}
    \vspace{-0.3cm}
    \State \[\,\,\,\W_m \gets \W_m - \eta\, \big(\W_m\,\yh_{m-1} - \a_m\big)\,\yh_{m-1}^\top\]
    \EndFor
    \EndFor
    \vspace{-0.18cm}
\Until{\,\,$\{\W_m\}$ not converged}\,\,
\end{algorithmic}
\end{algorithm}
\end{minipage}
\begin{minipage}{0.63\linewidth}
\begin{algorithm}[H]
\caption{LocoProp-M: LocoProp Using Matching Loss}\label{alg:locoprop-m}
\begin{algorithmic}
    \State \textbf{Input} weights $\{\W_m\}$ where $m \in [M]$ for an $M$-layer network, activation step size $\gamma$, weight learning rate $\eta$
\Repeat
    \State $\bullet$\, perform a \textbf{forward pass} fix the \emph{inputs} $\{\yh_{m-1}\}$
    \State $\bullet$\, perform a \textbf{backward pass} and set the \emph{MD targets} \vspace{-0.1cm}\[\y_{m} = \yh_m - \gamma\, \nabla_{\ah_m} L(\y, \yh)\]\vspace{-0.5cm}
    \For{each layer $m \in [M]$ \textbf{in parallel}}
    \vspace{0.05cm}
    \For{$T$ iterations}
    \vspace{-0.3cm}
    \State \[\,\,\W_m \gets \W_m - \eta\, \big(f_m(\W_m\,\yh_{m-1}) - \y_m\big)\,\yh_{m-1}^\top\]
    \EndFor
    \EndFor
    \vspace{-0.18cm}
\Until{\,\,$\{\W_m\}$ not converged}\,\,
\end{algorithmic}
\end{algorithm}
\end{minipage}
}
\vspace{-0.4cm}
\end{figure*}

\subsubsection{Local Matching Loss}
We now discuss another variant of LocoProp using the matching loss of each layer. Given the current post-activation $\yh_m = f_m(\W_m\,\yh_{m-1})$ at layer $m \in [M]$ with a non-decreasing transfer function $f_m$, we define the \emph{Mirror Descent (MD) target}
of the layer as $\y_m = \yh_m - \gamma\, \nabla_{\ah_m} L(\y, \yh)$ where $\gamma > 0$ is again the activation step size. The MD target corresponds to a MD step~\citep{mirror} on the current post-activations using the gradient of the final loss with respect to the pre-activation. Keeping the input to the layer $\yh_{m-1}$ fixed, our local optimization problem at layer $m \in [M]$ now consists of minimizing the matching loss between the new post-activation $f_m(\Wt\yh_{m-1})$ and the MD target, plus a similar squared $\mathrm{L}_2$-norm regularizer,
\begin{equation}
    \label{eq:matching_obj}
    D_{F_m}\!\big(\y_m, f_m(\Wt\yh_{m-1})\big) + \sfrac{1}{2\eta}\, \Vert \Wt - \W_m\Vert^2\, .
\end{equation}
Similarly, $\eta > 0$ controls the trade-off between minimizing the loss and the regularizer. The following proposition shows the convexity of the problem~\eqref{eq:matching_obj} in $\W_m$ in every layer $m \in [M]$ for a fixed input $\yh_{m-\!1}$ and \mbox{MD target $\y_m$.}

\vspace{-0.1cm}
\begin{restatable}{proposition}{propconvex}
\label{prop:convex}
Given fixed input $\yh_{m-1}$ and MD target $\y_m$, the optimization problem~\eqref{eq:matching_obj} is convex in $\W_m$ in every layer $m \in [M]$.
\end{restatable}
\vspace{-0.25cm}
Setting the derivative of the objective to zero, we can write $\W_m^{\new}$ as the solution of a fixed point iteration,
\begin{equation}
\label{eq:matching_fixed}
    \W^{\new}_m = \W_m - \eta\, \big(f_m(\W^{\new}_m\yh_{m-1}) - \y_m\big)\,\yh_{m-1}^\top\, .
\end{equation}
Notably, calculating the gradient in Eq.~\eqref{eq:matching_fixed} only requires the value of the input to the layer along with the difference between the post-activations. Unlike LocoProp-S, the fixed-point iteration of LocoProp-M in Eq.~\eqref{eq:matching_fixed} does not yield a closed-form solution in general. However, the following proposition provides the approximate preconditioned update form.
\begin{restatable}{proposition}{locompreform}
\label{prop:locom-pre-form}
The LocoProp-M update in Eq.~\eqref{eq:matching_fixed} can approximately be written in the vectorized form $\w_m = \vect \W_m$ as $\w_m^{\new} = \w_m + \bm{\delta}_m$,
\[
\bm{\delta}_m = -\eta_e\,\bm{C}_m^{-1} (\I \otimes \H_{F_m}^{-1}) \vect(\nabla_{\W_m}L(\y, \yh))\, ,
\]
where $\H_{F_m} = \nabla^2 F_m$ is the Hessian matrix and
\[
\bm{C}_m = \I \otimes \H_{F_m}^{-1} + \eta\, \yh_{m-1}\yh_{m-1}^\top \otimes \I\, ,
\]
is the preconditioner matrix.
\end{restatable}
Interestingly, similar to LocoProp-S, the first iteration of the procedure always recovers BackProp,
\begin{align*}
\W^{\new}_m \approx &\, \W_m - \eta\, \big(f_m(\W_m\yh_{m-1}) - \y_m\big)\,\yh_{m-1}^\top\\
= &  \,\W_m - \eta_e\, \nabla_{\ah_m}L(\y, \yh)\,\yh_{m-1}^\top\, . \tag{BackProp}
\end{align*}
Thus, any further iterations enhances the initial BackProp update towards the fixed-point solution of Eq.~\eqref{eq:matching_fixed}.
Also interestingly for a network which already utilizes a matching loss in the last layer, the objective in Eq.~\eqref{eq:matching_obj} with MD target and $\gamma = 1$ corresponds to directly minimizing the final loss.
\begin{restatable}{proposition}{propfinal}
\label{prop:final}
For networks with a matching loss  $L(\y, \yh) =  L_{f_M}\!(\y, \yh)$ in the last layer $M$, the LocoProp-M objective in Eq.~\eqref{eq:matching_obj} with $\gamma=1$ corresponds to the the loss of the network plus the regularizer term on the weights.
\end{restatable}
\vspace{-0.3cm}

Proposition~\ref{prop:final} shows why layerwise matching loss is a more natural choice for defining the local problems, as the last layer simply minimizes the final loss with respect to to the weights of the last layer. The LocoProp-S and LocoProp-M variants are given in Algorithm~\ref{alg:locoprop-s} and Algorithm~\ref{alg:locoprop-m}, respectively.

\section{EXPERIMENTS}

We perform an extensive study on optimizing a deep auto-encoder on three standard datasets: MNIST~\citep{lecun-mnisthandwrittendigit-2010}, Fashion MNIST~\citep{xiao2017}, and CURVES\footnote{Downloadable at \url{www.cs.toronto.edu/~jmartens/digs3pts_1.mat}.}. Our emphasis on the deep auto-encoder task is due to it being a standard benchmark when studying new second-order methods for optimization~\citep{martens2010deep, goldfarb2020practical}. Due to the limited size (only a few million parameters), it allows answering a plethora of questions with rigor on the effectiveness of our local loss construction. With our extensive tuning for the experiments, we find that adaptive first-order methods work far better than what is reported in the existing literature~\citep{goldfarb2020practical}.

For all experiments, the batch size is set to 1000 and the model is trained for 100 epochs with a learning rate schedule that includes a linear warmup for 5 epochs followed by a linear decay towards zero. The local iterations in LocoProp variants apply an additional linear decay schedule on the given learning rate. There are three configurations of autoencoders used in training: (a) Standard sized (2.72M parameters) with layer sizes: [1000, 500, 250, 30, 250, 500, 1000] (b) Deep (6.72M parameters)  [1000, 500 $\times$ 8 times, 250, 30, 250, 500 $\times$ 8 times, 1000] and (c) Wide (26M parameters)  [4000, 2000, 1000, 120, 1000, 2000, 4000]. We conduct several ablation studies on choices such as (a) transfer functions (ReLU, $\tanh$), (b) model sizes, (c) datasets, (d) number of local iterations, and (e) inner optimizer for local iterations. Our implementation is in TensorFlow~\citep{tensorflow} and all walltime measurements are made on a V100 GPU. For the second-order method implementation, K-FAC computes statistics based on the sampled gradients from the model distribution \citep{kunstner2019limitations}. For tuning, we use a Bayesian optimization package for each of our experiments. We search for hyper-parameters ($\eta, \beta_1, \beta_2, \epsilon$) for over hundreds (for larger scale models) to several thousand trials (for standard sized models). The search space, list of hyper-parameters, and further ablations across transfer functions, regularization, and batch sizes are available in the supplementary material, and the code to reproduce the experiments is available at \url{https://github.com/google-research/google-research/tree/master/locoprop}. We focus on our results on the standard, deep, and wide autoencoder variants with $\tanh$ transfer function on the MNIST dataset and relegate the additional results to the appendix.

\begin{figure*}[t!]
\vspace{-0.2cm}
\begin{center}
    \subfigure[RMSProp performs the best (for BackProp).]{\includegraphics[height=0.3\linewidth]{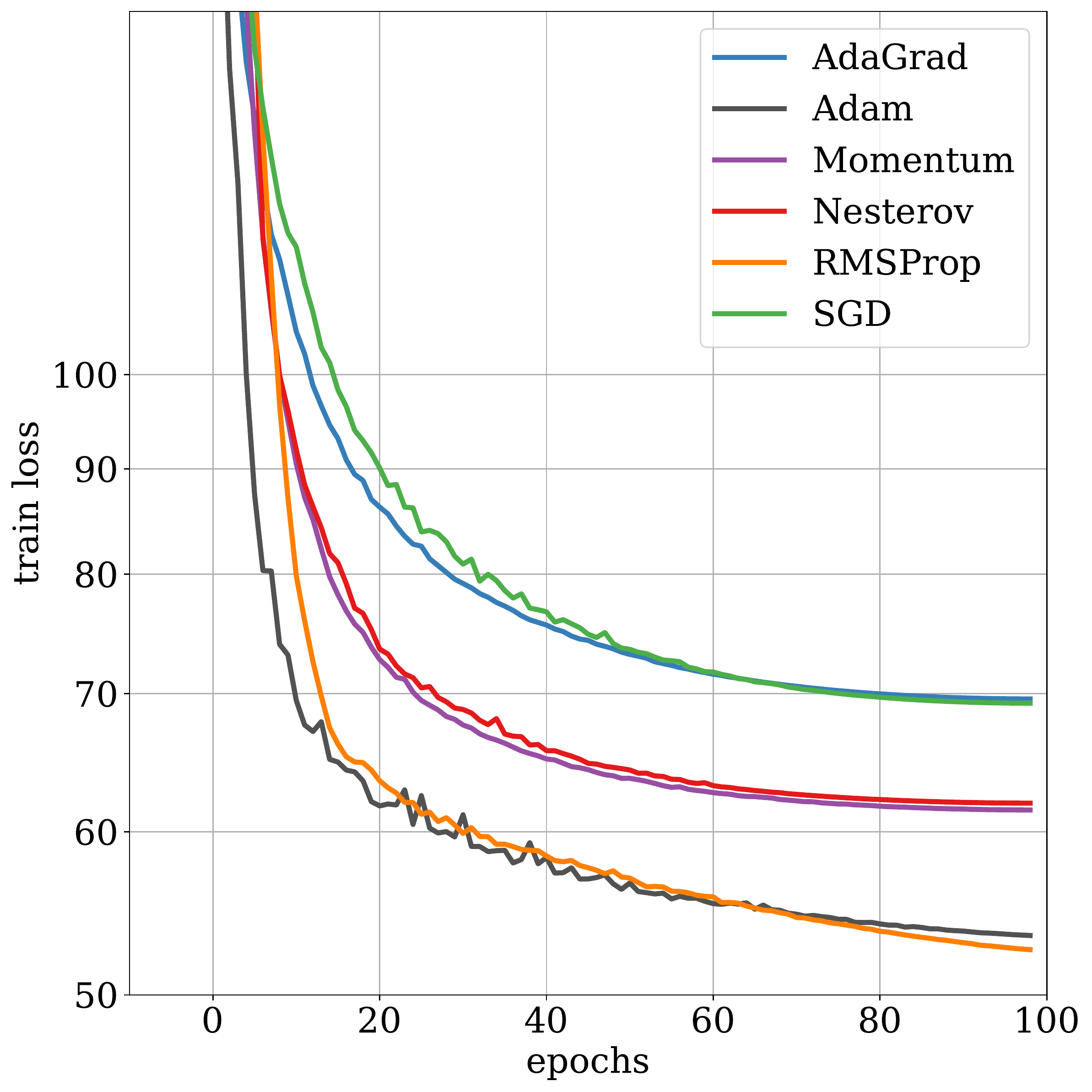}}  
    \hspace{0.1cm}    
    \subfigure[LocoProp `M' variant performs better than `S', at 10 local \mbox{iterations}]{\includegraphics[height=0.3\linewidth]{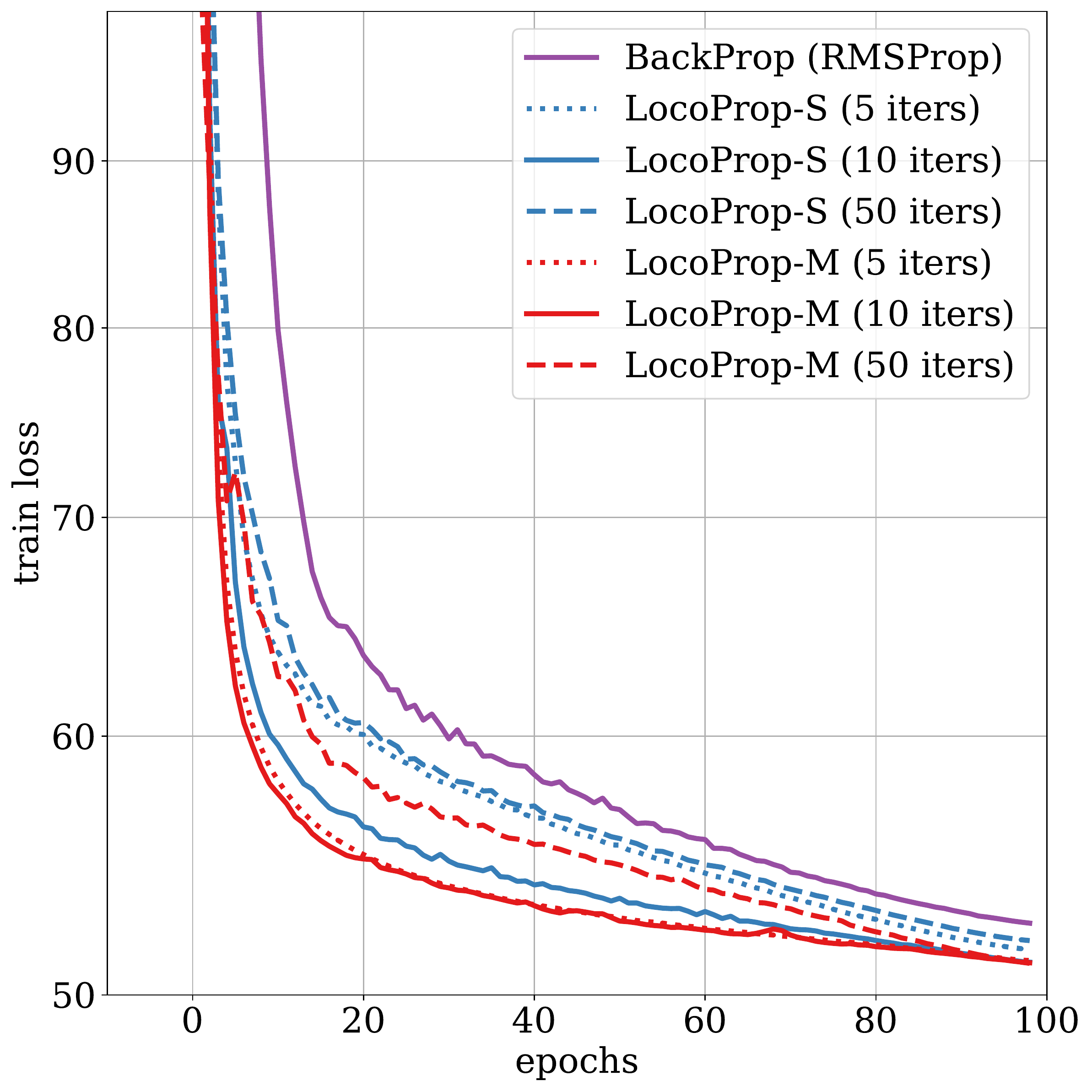}}
    \hspace{0.1cm} 
    \subfigure[RMSProp works best for local \mbox{iterations} with LocoProp-M]{\includegraphics[height=0.3\linewidth]{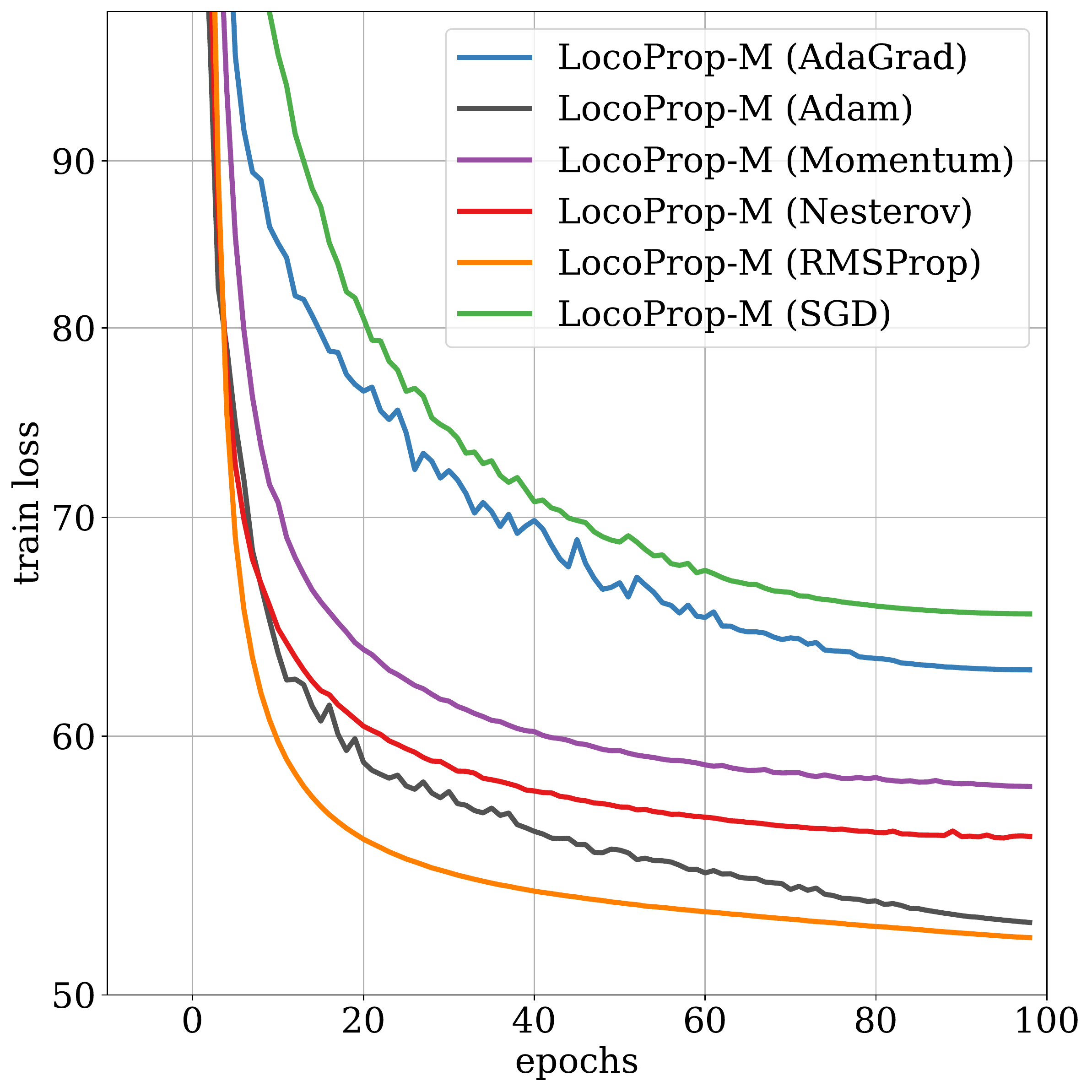}}

    \vspace{-0.25cm}
    \caption{Results on the MNIST dataset with $\tanh$ \mbox{transfer} function. Comparisons on the standard autoencoder: (a) first-order methods, (b)  LocoProp-S vs. LocoProp-M, (c) effect of inner optimizer on LocoProp-M.
    }
    \label{fig:mnist_tanh}
    \end{center}
    \vspace{-0.1cm}
\end{figure*}

\begin{table*}[t!]
\vspace{-0.2cm}
\caption{Additional computational and memory cost for a layer with dimensions $\{d_i\}$ and batch size $b$. The number of local iterations for LocoProp is denoted by $T$.}
\vspace{-0.4cm}
\begin{center}
\scalebox{0.85}{
\begin{tabular}{ c|c|c } 
 \toprule
 \textbf{Algorithm} & \textbf{Memory} &  \textbf{Computation} \\ 
 \midrule
 Adam & $\mathcal{O}\big( \prod_i d_i \big)$ & $ \mathcal{O}\big( \prod_i d_i\big)$  \\[1mm] 
  \hline
 Shampoo & $\mathcal{O}\big(\sum_i d_i^2\big)$ & $ \mathcal{O}\big(\sum_i  d_i^3\big)$  \\ [1mm]
 K-FAC & $\mathcal{O}\big(\sum_i d_i^2\big) $ & $\mathcal{O}\big(b  \sum_i d_i^2\big)  +  \mathcal{O}\big(\sum_i  d_i^3\big) + \mathcal{O}\big(b \prod_i d_i\big)$  \\ [1mm]
  \hline
 LocoProp & $\mathcal{O}\big( \prod_i d_i\big) + \mathcal{O}\big(b\sum_i d_i\big)$ & $ \mathcal{O}\big(b T \prod_i d_i\big)$  \\
 \bottomrule
\end{tabular}
}
\label{tbl:complexity}
\end{center}
\vspace{-0.5cm}
\end{table*}

\subsection{Tuned Results for First-order Methods}
We tune baseline first-order optimizers with a Bayesian optimization package for thousands of trials. We use standard implementations of optimizers in TensorFlow and tune all relevant hyper-parameters. 
The search space and further information are given in the  supplementary material. Firstly, we observe that the \mbox{RMSProp} optimizer works remarkably well, as seen in Figure~\ref{fig:mnist_tanh}(a). A difference to note is that RMSProp in TensorFlow includes an option for enabling momentum. We utilize this option, which is very similar to Adam: The former uses heavy ball momentum whereas the latter uses exponential moving averages. We discover that the performance gap between first-order and second-order methods, while still significant, is smaller than what has been previously reported in the literature. We attribute this to our exhaustive hyper-parameter tuning as well as the fact that we made use of of best practices for training neural networks such as learning rate warmup and decay.

\begin{figure*}[t!]
\vspace{-0.2cm}
\begin{center}
    \subfigure[Standard autoencoder]{\includegraphics[height=0.3\linewidth]{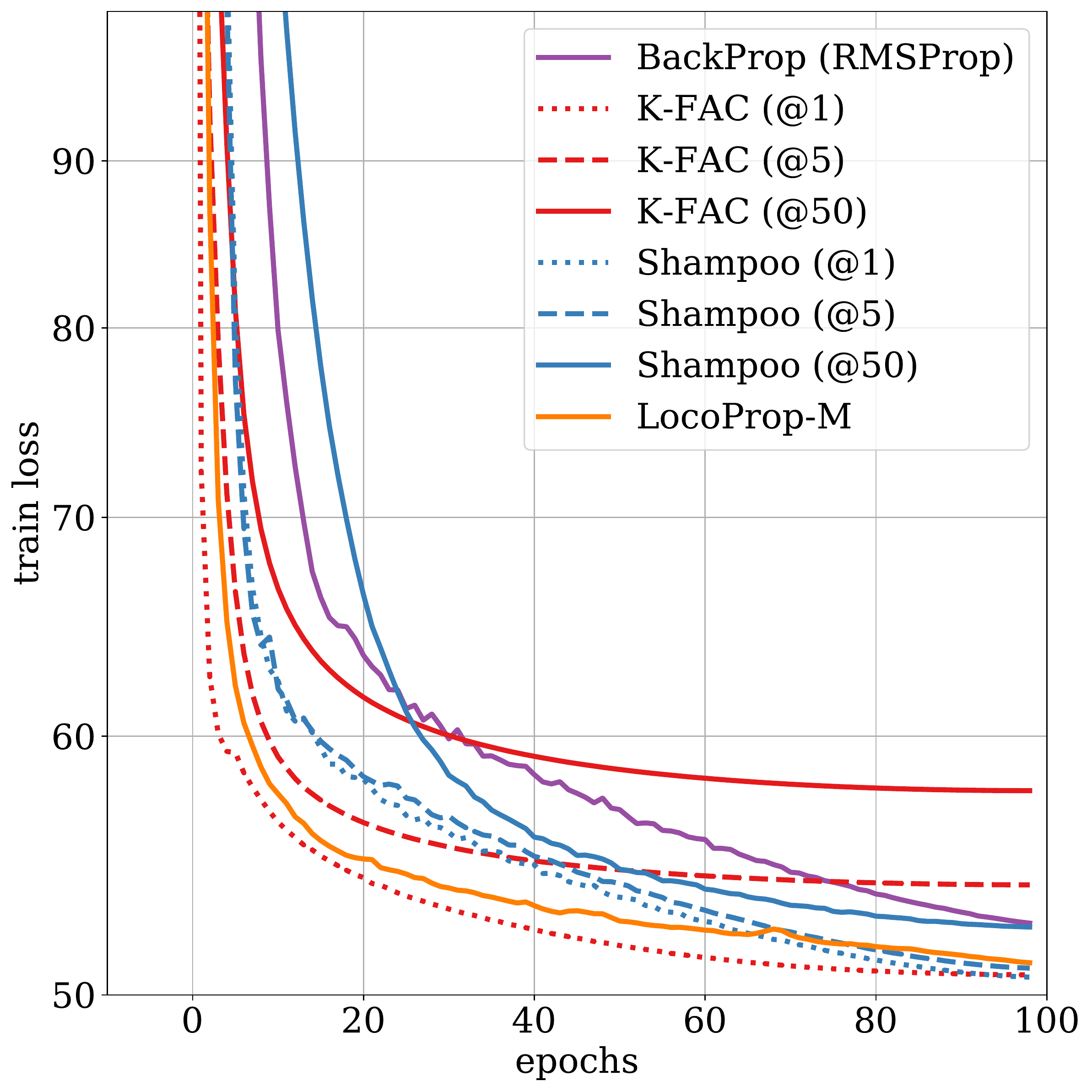}}  
    \hspace{0.1cm}    
    \subfigure[Wide autoencoder]{\includegraphics[height=0.3\linewidth]{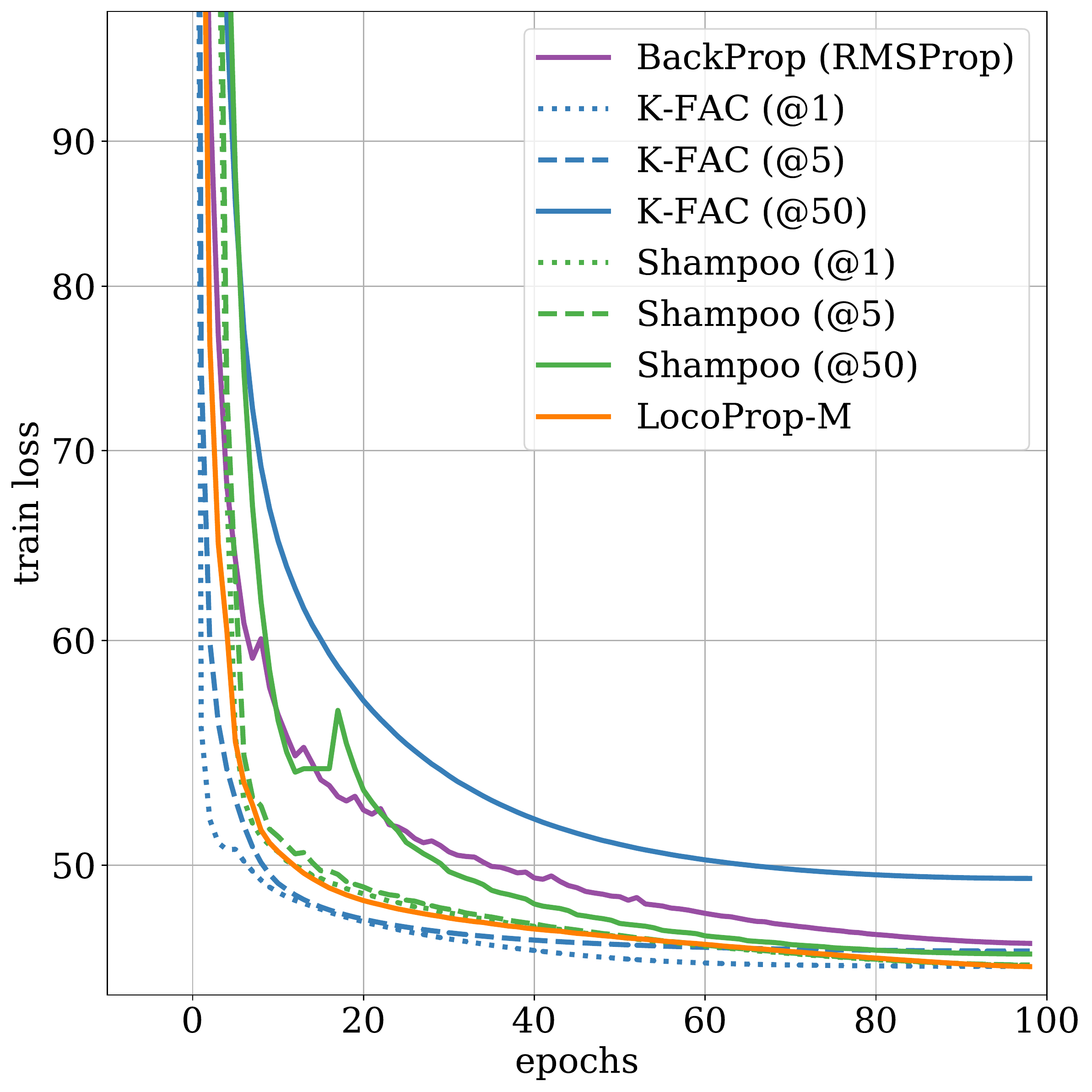}}  
    \hspace{0.1cm}    
    \subfigure[Deep autoencoder]{\includegraphics[height=0.3\linewidth]{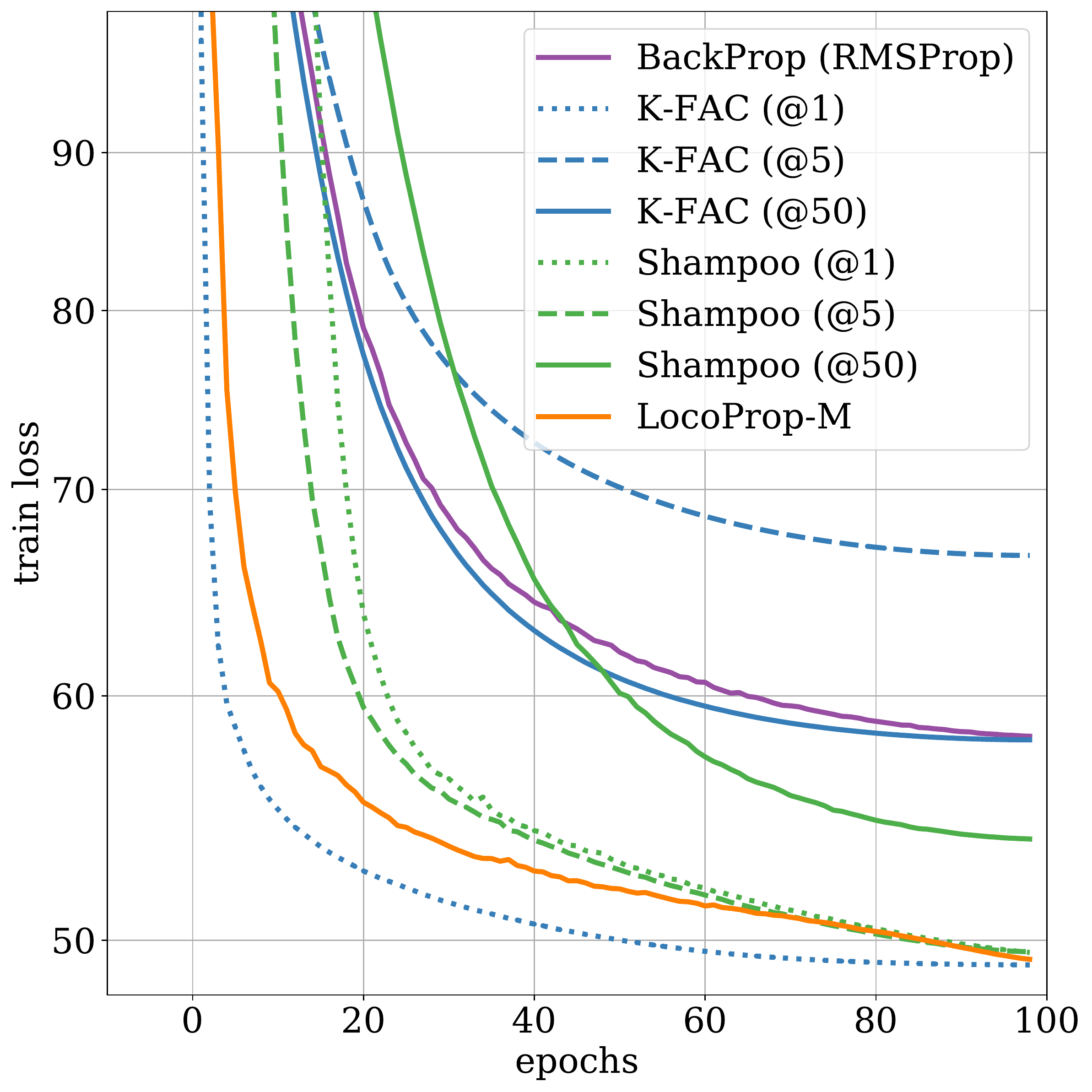}}
    \vspace{-0.2cm}
    \caption{Comparisons for standard, wide, and deep autoencoders on MNIST with $\tanh$ \mbox{transfer} function. All algorithms are trained for 100 epochs. @k indicates the interval for carrying out inverse ($p$th root) operation.}
    \label{fig:epoch_speedup}
    \vspace{-0.2cm}
    \end{center}
\end{figure*}

\begin{figure*}[]
\vspace{-0.3cm}
\begin{center}
    \subfigure[Standard autoencoder]{\includegraphics[height=0.3\linewidth]{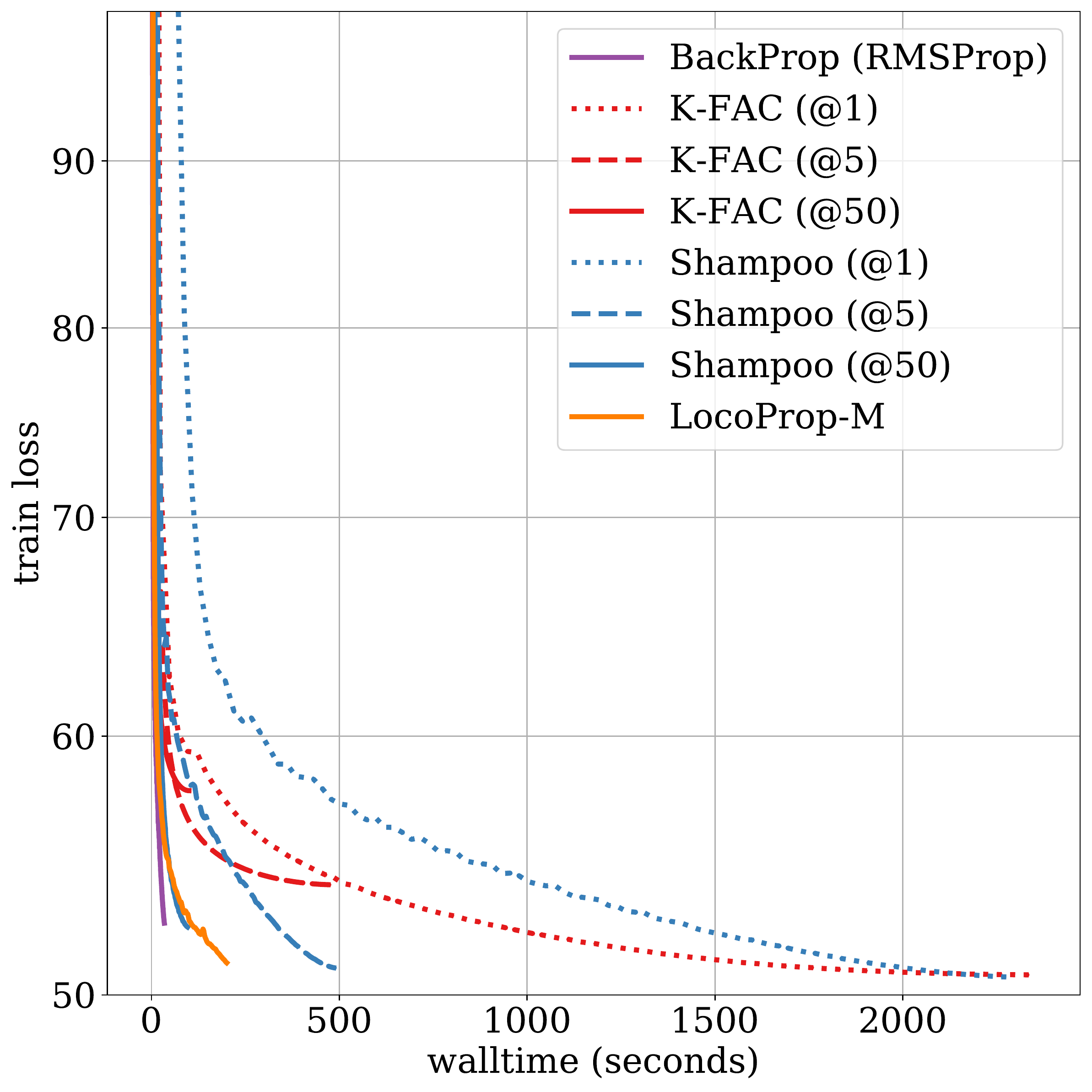}}  
    \hspace{0.1cm}    
    \subfigure[Wide autoencoder]{\includegraphics[height=0.3\linewidth]{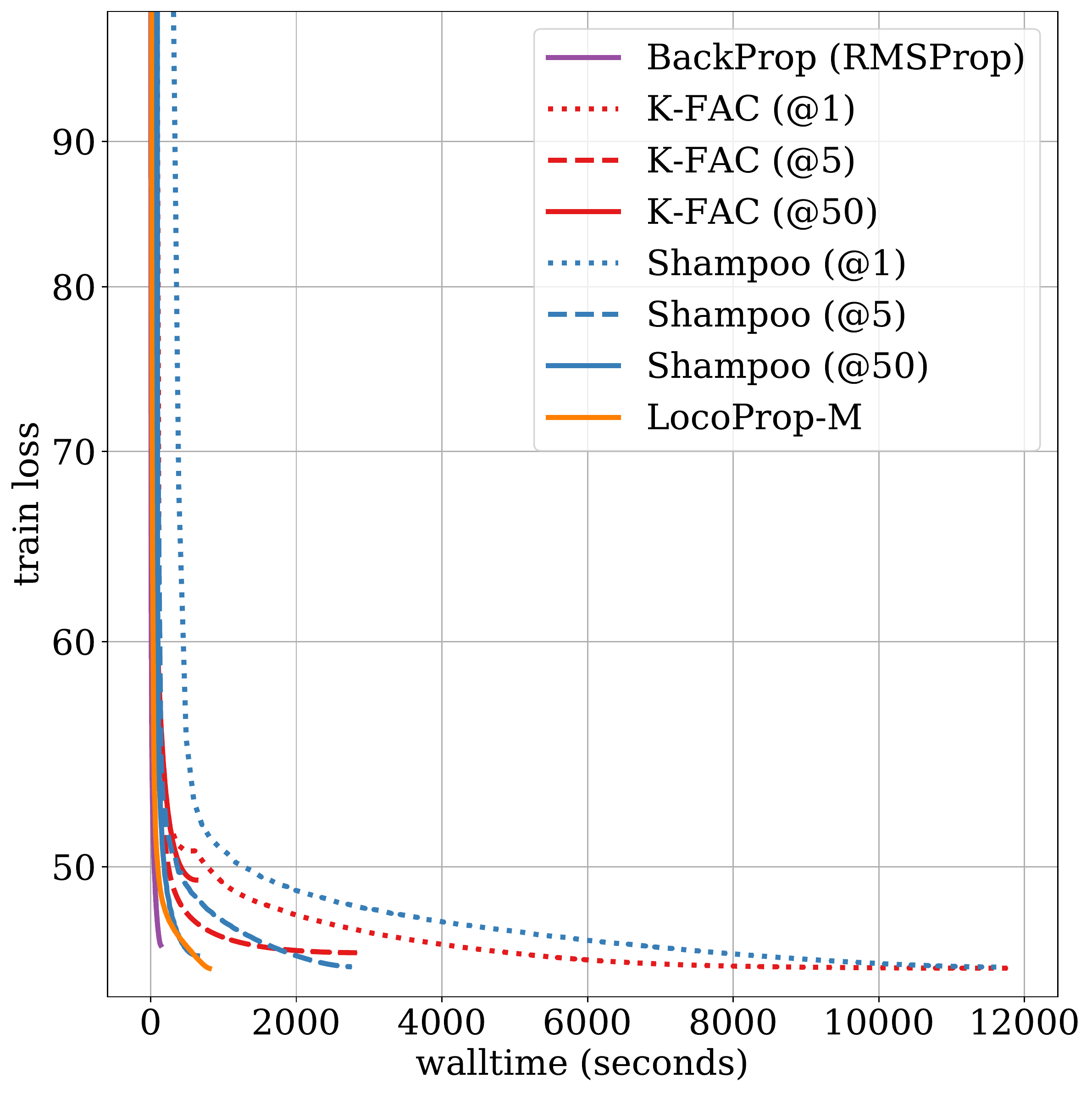}}  
    \hspace{0.1cm}    
    \subfigure[Deep autoencoder]{\includegraphics[height=0.3\linewidth]{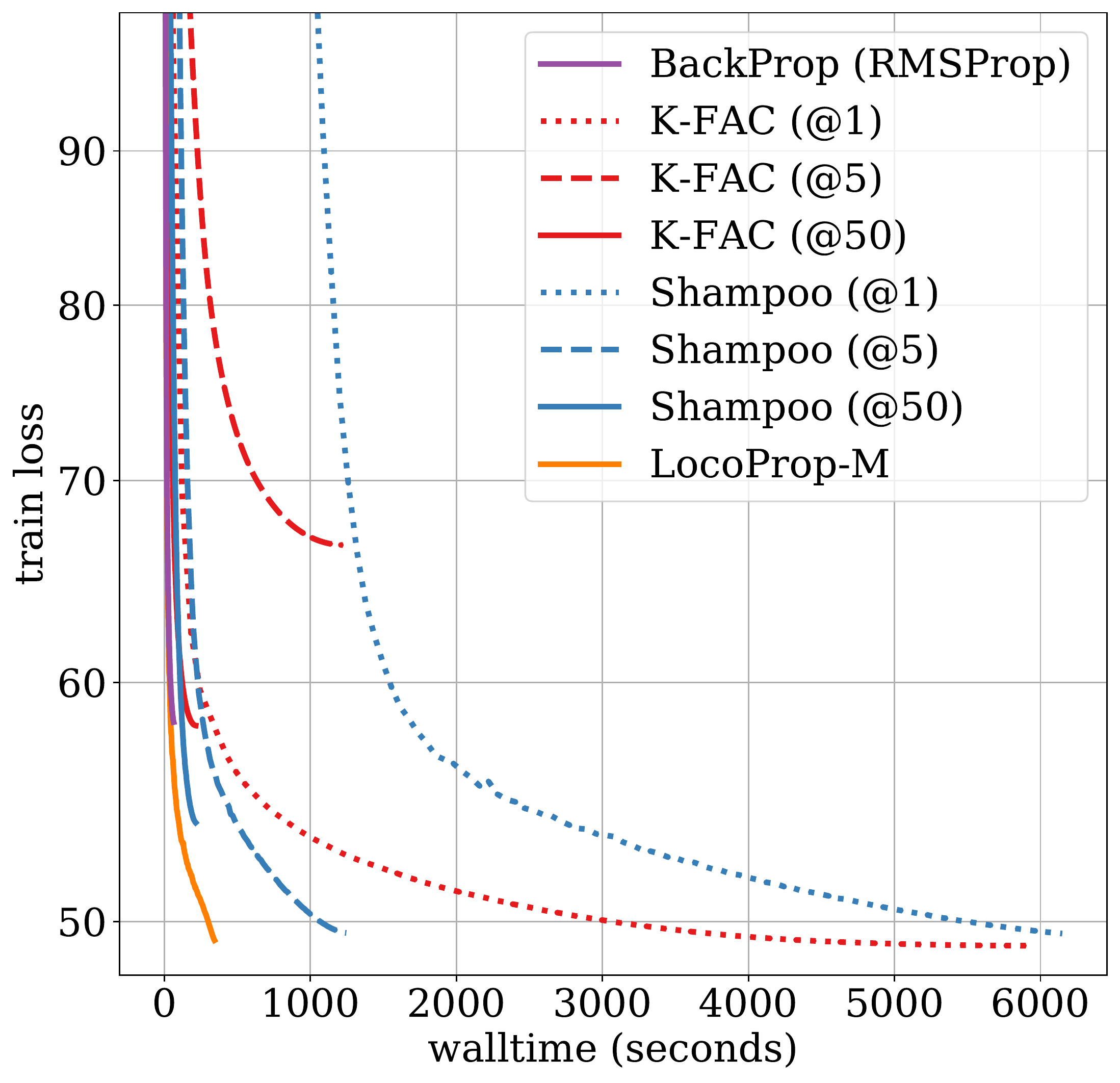}}
    \vspace{-0.2cm}
    \caption{Walltime comparisons for standard, wide, and deep autoencoders on MNIST with $\tanh$ \mbox{transfer} function. LocoProp is much faster than optimized second-order methods on a single V100 GPU (limited parallelism). @k indicates the interval for carrying out inverse ($p$th root) operation.}
    \label{fig:speedup}
    \vspace{-0.4cm}
    \end{center}
\end{figure*}


\vspace{-0.1cm}
\subsection{Comparison between LocoProp-S and LocoProp-M Variants}
We compare the performance of (a) LocoProp-S (using GD targets) and (b) LocoProp-M (using MD targets) on the standard sized deep encoder model on the MNIST dataset. We notice that the matching loss variant consistently performs better as seen in Figure~\ref{fig:mnist_tanh}(b). A rationale is that the matching losses are more adaptive to the local non-linear transfer functions ($\tanh$, in this case) than squared loss. 

\subsection{Number of Local Iterations and Choice of Optimizer for LocoProp-M}
The number of local gradient descent steps is a hyper-parameter for LocoProp. As described earlier, a single local iteration is equivalent to BackProp. Intuitively, further iterations are expected to improve on minimizing the local losses and make the update closer to a preconditioned gradient form. Thus, we expect to see more improvements as we increase the number of local iterations. However, additional improvements should diminish with the number of iterations. This is aligned with our observation in Figure~\ref{fig:mnist_tanh}(b), where increasing the number of local iterations from $5$ to $10$ and $50$ consistently improves the performance. However, the performance gain after $10$ local iterations is marginal, and thus, for all further experiments, we fix the number of local iterations to $10$. Consequently, we refer to LocoProp-M with $10$ local iterations as LocoProp. The LocoProp iterations minimize the local losses for each layer by gradient descent. We ablate across various choices of optimizers in Figure~\ref{fig:mnist_tanh}(c). Again, we find that RMSProp works quite well.

\subsection{Computational and Memory Complexity}
Here we provide a detailed account of the additional computational and memory complexity of LocoProp along with standard first-order and second-order methods in Table~\ref{tbl:complexity}. For notations, we use $b$ for batch size, and $\{d_i\}$ denotes the dimension of the weights of the layer, for example, $\{d_1, d_2\}$ for a fully connected network, and $T$ denotes the number of local iterations. K-FAC has an additional computational complexity compared to Shampoo due to calculating sampled gradients for the Fisher approximation. LocoProp has an advantage over KFAC and Shampoo both in terms of memory and computational complexity when batch size is comparable or smaller than the dimensions of the layers. Indeed we observe this in our experiments in Section~\ref{sec:comparison}.

\subsection{Comparison to Second-order Methods}
\label{sec:comparison}
We perform an end to end wall time comparison against second-order methods. LocoProp has similar convergence behavior to second-order methods while having significantly less overhead both in terms of memory and computation time. Second-order methods can be made faster by reducing the amount of computation by running the most expensive part of the step (i.e., matrix inverse or root) only every $k$ steps. Note this comes with degradation in overall quality which we notice especially for K-FAC. The effectiveness of delayed inverse calculation depends on task specific parameters such as batch size and number of examples. For larger batch sizes, every step makes large progress and the curvature estimates need to be updated more frequently. Several investigations along these lines have studied the effect of stale preconditioners~\citep{anilpractical, osawa18, ba2016distributed}. We find that LocoProp has a significant wall-time advantage compared to second-order methods. Second-order methods barely match the walltime performance of LocoProp by running the inverse $p$-th roots every 50 steps for Shampoo whereas for K-FAC delayed preconditioning configuration degrades the solution quality as seen in Figure~\ref{fig:speedup} and the corresponding steps to convergence in Figure~\ref{fig:epoch_speedup}.

\section{Conclusions} 
We presented a simple yet effective enhancement to BackProp via a local loss construction that reduces the gap between first-order and second-order optimizers. We perform a critical study of its effectiveness across a wide range of model sizes and dataset choices. The construction is embarrassingly parallel, allowing it to have a wallclock time advantage. Finally, we also have shown that our variant with squared (matching) loss is (approximately) an iterative version of a preconditioned update. Future work involves scaling the method up to much larger architectures across tasks.

{\bf Broader Impact and Limitations.} This paper introduces a technique that could be used to accelerate the training of neural networks. This could have a positive downstream implication on reducing the energy usage of training large models. LocoProp is a new technique; it still remains to be seen how well the method generally works across tasks. 

\section*{Acknowledgments}
We would like thank Frank Nielsen and our anonymous reviewers for their valuable feedback.

\bibliography{refs}

\newpage
\appendix
\onecolumn \makesupplementtitle
\section{MATCHING LOSSES OF COMMON TRANSFER FUNCTIONS}
\label{app:transfer}
We provide a more extensive list of convex integral functions of common transfer functions in Table~\ref{tab:transfer-ext}.
 \renewcommand{\arraystretch}{1.5}
 \begin{table*}[h!]
\caption{Example of elementwise non-decreasing transfer functions and their corresponding convex integral function. The Bregman divergence can be formed by plugging in the convex integral function into the definition: $D_F(\ah, \a) = F(\ah) - F(\a) - f(\a)^\top (\ah - \a)$. For instance, for ``linear'' activation, we have $F_{\text{\tiny Sq}}(\a) =  \sfrac{1}{2}\,\Vert\a\Vert^2$ and the Bregman divergence is $D_{F_{\text{\tiny Sq}}}(\ah, \a) = \sfrac{1}{2}\,\Vert\ah\Vert^2 - \sfrac{1}{2}\,\Vert\a\Vert^2 - \a^\top (\ah - \a) = \sfrac{1}{2}\,\Vert\ah - \a\Vert^2$. 
For non-decreasing but not strictly increasing transfer functions such as ``ReLU'', the induced divergence does not satisfy the necessity condition in Property (III):\, $D_F(\ah, \a) = 0$ if $\ah = \a$ (but not only if).}
\label{tab:transfer-ext}
\vspace{-0.25cm}
\begin{center}
\begin{small}
\begin{sc}
\resizebox{1\textwidth}{!}{
\begin{tabular}{lccc}
\toprule
Name & Transfer Function $f(\a)$ & Convex Integral Function $F(\a)$ & Note\\
\midrule
Step Function & $\sfrac{1}{2}\,(1 + \sign(\a))$ & $\sum_i \max(a_i, 0)$ & --\\
Linear & $\a$ & $\sfrac{1}{2}\,\Vert\a\Vert^2$ & --\\
(Leaky) ReLU & $\max(\a, 0) -\beta \max(-\a, 0)$& $\sfrac{1}{2}\,\sum_i a_i \big(\max(a_i, 0) -\beta \max(-a_i, 0)\big)$ & $\beta\geq 0$\\
Sigmoid & $(1 + \exp(-\a))^{-1}$ & $\sum_i\big(a_i + \log(1 + \exp(-a_i))\big)$& --\\
Softmax & $\sfrac{\exp(\a)}{\sum_i \exp(a_i)}$ & $\log\sum_i\exp(a_i)$ & --\\
Hyperbolic Tan & $\tanh(\a)$ & $\sum_i\log\cosh(a_i)$& -- \\
Arc Tan & $\arctan(\a)$ & $\sum_i \big(a_i\arctan(a_i) - \log\sqrt{1 + a_i^2}\big)$& -- \\
SoftPlus & $\log(1 + \exp(\a))$ & $-\sum_i\mathrm{Li}_2(-\exp(a_i))$ & $\mathrm{Li}_2\coloneqq$ Spence's func.\\
ELU & $[f(\a)]_i = \begin{cases} a_i &  a_i\geq 0\\ \beta(\exp a_i - 1) & \text{otherwise} \end{cases}$ & $\sum_i \big(\sfrac{a_i^2}{2}\, \mathbb{I}(a_i\geq 0) + \beta(\exp a_i - a_i - 1)\big)\, \mathbb{I}(a_i <  0)\big) $ & $\beta\geq 0$\\
\bottomrule
\end{tabular}
}
\end{sc}
\end{small}
\end{center}
\vskip -0.2in
\end{table*}
\section{MISSING PROOFS}

\propconvex*
\begin{proof}
The proof simply follows from the construction of the matching loss; that is, the convexity of Bregman divergence $D_{F_m}\!(\ah_m, \a_m) = L_{f_m}\!(\W_m\,\yh_{m-1}, f_m^{-1}(\y_m))$ in the first argument $\ah_m$ and the fact that $\ah_m$ is a linear function of $\W_m$, i.e., $\ah_m = \W_m\, \yh_{m-1}$.
\end{proof}

\locompreform*
\begin{proof}
Consider the LocoProp-M update:
\begin{align}
    \W^{\new}_m & = \W_m - \eta\, \big(f_m(\W^{\new}_m\yh_{m-1}) - \y_m\big)\,\yh_{m-1}^\top\, .\nonumber
\intertext{Writing $\W^{\new}_m = \W_m + \Delt_m$, we have}
    \Delt_m & = - \eta\, \big(f_m((\W_m + \Delt_m)\yh_{m-1}) - \y_m\big)\,\yh_{m-1}^\top\label{eq:matching_fixed-delta}
\intertext{Assuming that $\Vert\Delt_m\Vert \ll \Vert\W_m\Vert$, we can approximate the rhs as}
    \Delt_m &\approx - \eta\, \big(f_m(\W_m \yh_{m-1}) + \H_{F_m}\Delt_m\yh_{m-1} - \y_m\big)\,\yh_{m-1}^\top\nonumber\\
    & = -\eta\, \H_{F_m}\Delt_m\yh_{m-1}\yh_{m-1}^\top - \eta_e\nabla_{\ah_m}L(\y, \yh)\yh_{m-1}^\top\nonumber
    \intertext{where $\H_{F_m} \coloneqq \nabla^2 F_m$ is the Hessian. Rearranging the terms, we can write}
    \Delt_m & + \eta\, \H_{F_m}\Delt_m\yh_{m-1}\yh_{m-1}^\top = -\eta_e\nabla_{\W_m}L(\y, \yh)\nonumber
    \end{align}
    \begin{align}
    \intertext{Or alternatively}
    & \H_{F_m}^{-1}\Delt_m + \eta\,\Delt_m\yh_{m-1}\yh_{m-1}^\top = -\eta_e\,\H_{F_m}^{-1}\nabla_{\W_m}L(\y, \yh)\label{eq:matching_fixed-sylv}
    \intertext{Eq.~\eqref{eq:matching_fixed-sylv} corresponds to a Sylvester equation~\citep{birkhoff2017survey}. The solution to this equation can be written in the vectorized form as}
    \bm{\delta}_m & = \vect \Delt_m = -\eta_e\,(\I \otimes \H_{F_m}^{-1} + \eta\, \yh_{m-1}\yh_{m-1}^\top \otimes \I)^{-1}\, (\I \otimes \H_{F_m}^{-1}) \vect(\nabla_{\W_m}L(\y, \yh))
\end{align}
Note that LocoProp-S assumes a flat geometry for which $\H_{F_m} = \I_d$ and Eq.~\eqref{eq:matching_fixed-sylv} reduces to the implicit gradient update obtained from the closed-form solution of LocoProp-S.
\end{proof}

\propfinal*
\begin{proof}
For a MD target with $\gamma=1$ at the last layer $M$, we have
\[
\y_M = \yh - \nabla_{\ah_M} L_{f_M}\!(\y, \yh) = \yh  \;-\!\!\!\!\!\!\!\!\!\! \underbrace{(\yh - \y)}_{\text{gradient of matching loss}}\!\!\!\!\!\!\!\!\!\!=\; \y \, .
\]
\vspace{-0.5cm}

Thus, the objective in Eq.~\eqref{eq:matching_obj} corresponds to minimizing the final loss $L_{f_M}\!(\y, \yh)$ plus the regularizer term with respect to $\W_M$.
\end{proof}

\section{LOCOPROP-S RECOVERS PROXPROP}
ProxProp~\citep{proximal} is motivated by approximating the proximal mapping~\citep{moreau1965proximite} of a function, which is commonly known as the implicit update~\citep{hassibi,pnorm}. ProxProp is motivated by first forming updated pre and post-activations recursively starting from the layer $M\!-\!1$ (one before the last layer):
\[
    \yh_{M-1}^+ = \yh_{M-1} - \tau\,\nabla_{\yh_{M-1}} L(\y, \yh)\, ,
\]
and for the layers below as,
\begin{align*}
    \ah_{m}^+ & = \ah_m - \nabla_{\ah_m} f_m(\ah_m) \big(f_m(\ah_m) - \yh_{m}^+\big)\, , \,\text{ for\, } m\in [M\!-\!1]\, ,\\
    \yh_{m-1}^+ & = \yh_{m-1} - \nabla_{\yh_{m-1}} \Big(\frac{1}{2}\Vert \W_m\yh_{m-1} - \ah^+_m\Vert^2\Big)\,, \,\text{ for\, } m \in [M\!-\!1] - \{1\}\, .
\end{align*}
The ProxProp then proceeds by first updating the weights of the last layer using gradient descent,
\[
\W_M^{\new} = \W_M - \tau\,\nabla_{\W_M} L(\y, \yh)\, .
\]
The remaining weights are then updated by as minimizing a squared loss between the current pre-activations and the updated pre-activations plus a squared regularizer terms on the wights:
\[
\W^{\new}_m = \argmin_{\Wt}\big\{\sfrac{1}{2}\, \Vert\Wt\yh_{m-1} - \ah^+_m\Vert^2 + \sfrac{1}{2\tau_m}\, \Vert \Wt - \W_m\Vert^2\big\}\, ,\, m \in [M\!-\!1]\,.
\]
At first glance, due to the convoluted formulation of the updated activations, it is unclear what the objective of ProxProp is chasing after. Although not mentioned in the original ProxProp paper, after some simplification, we can rewrite the updated activations as a gradient step on the current values of the activations:
\begin{align*}
    \yh_m^+ & = \yh_m - \tau\,\nabla_{\yh_M} L(\y, \yh)\, ,\\
    \ah_m^+ & = \ah_m - \tau\,\nabla_{\ah_M} L(\y, \yh)\,,\, \text{ for\, } m\in [M\!-\!1]\, .
\end{align*}
With this simplification, we can recognize that the updated activations (i.e., the targets, in our terminology) can in fact be formed in parallel for all layers after the backward pass. Thus, the LocoProp-S and ProxProp formulations become equivalent except for updating the weights of the last layer; ProxProp applies vanilla gradient descent while LocoProp-S uses a local loss function for updating these weights as well, which recovers ProxProp when the number of local iterations is one.

This example further illustrates how our LocoProp formulation, which minimizes a loss to a target plus a regularizer on the weights, simplifies the construction and allows developing several extensions by adjusting each component.

\section{LOCOPROP-S CLOSED-FORM UPDATE RESEMBLES K-FAC PRECONDITIONERS}
\label{app:natural}
There are several types of preconditioners such as K-FAC and Shampoo that one could use to train neural networks. In this section, we make a connection to the K-FAC \citep{martens2015optimizing} update rule, which applies a Kronecker factored approximation to the Fisher information matrix~\citep{amari}. Let $\w$ be the vectorized form of matrix $\W$. Recall that the Fisher information matrix is defined as
\begin{equation}
    \label{eq:fisher}
    \bm{F}(\w)  =  \mathbb{E}_{\substack{\textbf{ $ \x \thicksim$ data} \\ \textbf{$\y \thicksim p(\y|\x,\w)$}}}\!\! \left[ \nabla_{\w}   L(\y,\yh)  \nabla_{\w}  L(\y,\yh)^\top \right]\, ,
\end{equation}
where the expectation is over the data and the model's predictive distribution. The natural gradient update~\citep{amari} is then defined as
\begin{equation}
    \label{eq:ngd}
\w^+ = \w - \eta\, \bm{F}(\w)^{-1}\nabla_{\w} L(\y,\yh)\, .
\end{equation}

For a single fully connected layer $\W_m \in \R^{d\times n}$ where $m\in[M]$, the gradient, denoted by $\bm{G}_m\in\R^{d\times{}n}$ can be obtained via the chain
rule as $\bm{G}_m = \nabla_{\ah_m}  L(\y, \yh)\, \yh_{m-1}^\top$, which in the vectorized form can be written as: $
\nabla_{\ah_m}  L(\y, \yh) \otimes \yh_{m-1} $.
We can then write the Fisher information matrix as
\begin{align*}
   {\bm{F}(\w_m)} &= \mathop{\mathbb{E}} \left[ (\nabla_{\ah_m}  L(\y, \yh) \otimes \yh_{m-1} ) \, ( \nabla_{\ah_m}  L(\y, \yh) \otimes \yh_{m-1} )^\top \right]\\
    &=  \mathop{\mathbb{E}} \left[ (\nabla_{\ah_m}  L(\y, \yh)\,  \nabla_{\ah_m}  L(\y, \yh)^\top ) \otimes \, ( \yh_{m-1}\, \yh_{m-1}^\top) \right]\, .
\end{align*}
Assuming independence between $ \nabla_{\ah_m}  L(\y, \yh) $ and $\yh_{m-1}$, K-FAC rewrites the Fisher in a tractable form as
\[
    {\bm{F}(\w_m)}
    \approx \mathop{\mathbb{E}} \left[ (\nabla_{\ah_m}  L(\y, \yh)\,  \nabla_{\ah_m}  L(\y, \yh)^\top ) \right] \otimes \, \mathop{\mathbb{E}} \left[\yh_{m-1}\, \yh_{m-1}^\top \right] \, .
\]

Let $\bm{D} \doteq \mathbb{E} \left[ (\nabla_{\ah_m}  L(\y, \yh)\,  \nabla_{\ah_m}
 L(\y, \yh)^\top ) \right]$ and $\bm{X} \doteq \mathbb{E} \left[\yh_{m-1}\, \yh_{m-1}^\top
\right]$. Then the K-FAC update rule can be simplified as:
\begin{equation}
    \label{eq:kfac}
  \W_m^{+} \approx \W_m - \eta\, \bm{D}^{-1} \bm{G}_m \bm{X}^{-1} \, .
\end{equation}
The prototypical implementation of K-FAC uses the moving average of the statistics over training batches. In Section~\ref{sec:locos}, we noted that how the exact solution of the squared loss in Eq.~\eqref{eq:squared_obj} emerges as a preconditioner that is similar to the right preconditioner of K-FAC in Eq.~\eqref{eq:kfac}, formed by the outer-product of the input activations. (In practice, constant diagonal entries are added to matrix $\bm{X}$ to avoid numerical issues.) Interestingly, one could similarly recover the left preconditioner as well by setting the target pre-activations using a \emph{natural gradient} step on the pre-activations:
\[
    \a_m = \ah_m - \gamma\, \bm{D}^{-1} \nabla_{\ah_m} L(\y, \yh)\, ,
\]
where $\bm{D}$ acts as the Fisher Information matrix treating the pre-activations as parameters, in which case the expectation is conditioned on $\x$ and is taken over the model's  predictive distribution. Substituting for the targets yields the closed-form update rule for LocoProp-S with \emph{natural gradient descent targets}:
\begin{equation}
    \W_m^+ = \W_m - \eta_e\,\bm{D}^{-1}\,\nabla_{\W_m}  L(\y, \yh)\,\big(\I + \eta\,\yh_{m-1}\yh_{m-1}^\top\big)^{-1}\, .
\end{equation}
Our exposition here is to make the connection which may aid us in analyzing the properties of LocoProp in the future.

\begin{figure*}[t]
\vspace{-0.7cm}
\centering
\scalebox{0.78}{
    \begin{minipage}{0.63\linewidth}
\begin{algorithm}[H]
\caption{PocoProp-S: PocoProp Using Squared Loss}\label{alg:pocoprop-s}
\begin{algorithmic}
    \State \textbf{Input} weights $\{\W_m\}$ where $m \in [M]$ for an $M$-layer network, activation step size $\gamma$, weight learning rate $\eta$
\Repeat
    \State $\bullet$\, perform a \textbf{forward pass} and fix the \emph{inputs} $\{\yh_{m-1}\}$
    \State $\bullet$\, perform a \textbf{backward pass} and set the \emph{post GD targets}
    \vspace{-0.1cm}\[\y_{m} = \yh_m - \gamma\, \nabla_{\yh_m} L(\y, \yh)\]\vspace{-0.5cm}
    \For{each layer $m \in [M]$ \textbf{in parallel}}
    \vspace{0.05cm}
    \For{$T$ iterations}
    \vspace{-0.2cm}
    \State \[\W_m \gets \W_m - \eta\bm{J}_{f_m}^\top\big(f_m(\W_m\yh_{m-1}) - \y_m\big)\yh_{m-1}^\top\]
    \EndFor
    \EndFor
    \vspace{-0.18cm}
\Until{\,\,$\{\W_m\}$ not converged}\,\,
\end{algorithmic}
\end{algorithm}
\end{minipage}
\begin{minipage}{0.63\linewidth}
\begin{algorithm}[H]
\caption{PocoProp-M: PocoProp Using Matching Loss}\label{alg:pocoprop-m}
\begin{algorithmic}
    \State \textbf{Input} weights $\{\W_m\}$ where $m \in [M]$ for an $M$-layer network, activation step size $\gamma$, weight learning rate $\eta$
\Repeat
    \State $\bullet$\, perform a \textbf{forward pass} fix the \emph{inputs} $\{\yh_{m-1}\}$
    \State $\bullet$\, perform a \textbf{backward pass} and set the \emph{dual MD targets} \vspace{-0.1cm}\[\a_{m} = \ah_m - \gamma\, \nabla_{\yh_m} L(\y, \yh)\]\vspace{-0.5cm}
    \For{each layer $m \in [M]$ \textbf{in parallel}}
    \vspace{0.05cm}
    \For{$T$ iterations}
    \vspace{-0.2cm}
    \State \[\W_m \gets \W_m - \eta\, \bm{J}_{f_m}^\top\big(\W_m\yh_{m-1} - \a_m\big)\,\yh_{m-1}^\top\]
    \EndFor
    \EndFor
    \vspace{-0.18cm}
\Until{\,\,$\{\W_m\}$ not converged}\,\,
\end{algorithmic}
\end{algorithm}
\end{minipage}
}
\vspace{-0.1cm}
\end{figure*}

\section{ALTERNATE LOCOPROP FORMULATIONS}
We introduce two additional formulations of LocoProp based on the post-activations, for which we defer the experimental evaluation to future work. Similarly, the first formulation is based on the squared loss whereas the second formulation uses the dual form of the Bregman divergence used in LocoProp-M. In both cases, a single local iteration recovers BackProp. These two variants are given in Algorithm~\ref{alg:pocoprop-s} and~\ref{alg:pocoprop-m}.

\subsection{A Variant of LocoProp-S Based on Post-activations}
We define a LocoProp-S variant by first setting the targets as $\y_m = \yh_m - \gamma\, \nabla_{\yh_m} L(\y, \yh)$ for $m \in [M]$, where $\gamma > 0$ is the activation step size. We refer to these targets as \emph{post GD targets}. Next, the local optimization problem is defined as
\begin{equation}
    \label{eq:post_squared_obj}
    \W^{\new}_m = \argmin_{\Wt}\, \big\{\sfrac{1}{2}\, \Vert f_m(\Wt\yh_{m-1}) - \y_m\Vert^2 + \sfrac{1}{2\eta}\, \Vert \Wt - \W_m\Vert^2\big\}\, .
\end{equation}
Similarly, $\eta > 0$ controls the trade-off between minimizing the loss and the regularizer. By setting the derivative of the objective to zero, the fixed point iteration for $\W_m^+$ can be written as,
\begin{equation}
\label{eq:post_gd_fixed}
    \W^{\new}_m = \W_m - \eta\, \bm{J}_{f_m}\!(\W^{\new}_m\yh_{m-1})^\top\big(f_m(\W^{\new}_m\yh_{m-1}) - \y_m\big)\,\yh_{m-1}^\top\, ,
\end{equation}
where $\bm{J}_{f_m}\!(\ah) = \frac{\partial f_m(\ah)}{\partial \ah}$ is the Jacobian of the transfer function $f_m$ (which is many cases is a diagonal matrix). Note that for the choice of post GD targets, a single fixed point iteration again recovers vanilla BackProp,
\begin{align}
\W^{\new}_m & \approx \W_m - \eta\, \bm{J}_{f_m}\!(\W_m\yh_{m-1})^\top\big(f_m(\W_m\yh_{m-1}) - \y_m\big)\,\yh_{m-1}^\top\nonumber\\
& = \W_m - \eta\, \bm{J}_{f_m}\!(\W_m\yh_{m-1})^\top\big(f_m(\W_m\yh_{m-1}) - (f_m(\W_m\yh_{m-1}) - \gamma\, \nabla_{\yh_m} L(\y, \yh))\big)\,\yh_{m-1}^\top\nonumber\\
& = \W_{m} - \eta\, \gamma\,
    \bm{J}_{f_m}\!(\W_m\yh_{m-1})^\top\nabla_{\yh_m} L(\y, \yh)\,\,\yh_{m-1}^\top\nonumber\\ & = \W_{m} - \eta_e\,
    \frac{\partial L(\y, \yh)}{\partial
       \ah_m}\frac{\partial \ah_m}{\partial \W_m}\nonumber\tag{BackProp}\,,
\end{align}
with an effective learning rate of $\eta_e = \eta\, \gamma > 0$. However in this case, the fixed point iteration in Eq.~\eqref{eq:post_gd_fixed} does not yield a closed-form solution in general and therefore, should be solved iteratively. We refer to this local loss construction variant as \emph{P\sout{ost L}ocoProp-S} for short. Other variants of \mbox{PocoProp-S} can be obtained by replacing the post GD targets with post natural GD targets with respect to the post activations. We defer the analysis of such variants to future work.

\subsection{A Variant of LocoProp-M Based on Post-activations}
Similarly, we define the \emph{dual MD targets} as $\a_m = \ah_m - \gamma\, \nabla_{\yh_m} L(\y, \yh)$ for $m \in [M]$, where $\gamma > 0$ is the activation step size. The local optimization problem is then defined using the dual form of the Bregman divergence used in Eq.~\eqref{eq:matching_obj}, that is 
\begin{equation}
    \label{eq:post-matching_obj}
    \W^{\new}_m = \argmin_{\Wt}\, \big\{D_{F^*_m}\!\big(f_m(\Wt\yh_{m-1}), f_m(\a_m)\big) + \sfrac{1}{2\eta}\, \Vert \Wt - \W_m\Vert^2\big\}
\end{equation}
Similarly, $\eta > 0$ controls the trade-off between minimizing the loss and the regularizer. Setting the derivative of the objective to zero, we can write $\W_m^+$ as the solution of a fixed point iteration,
\begin{align}
\label{eq:post-matching_fixed}
    \W^{\new}_m & = \W_m - \eta\, \bm{J}_{f_m}\!(\W^{\new}_m\yh_{m-1})^\top\big(f^*_m(f_m(\W^{\new}_m\yh_{m-1})) - f^*_m(f_m(\a_m))\big)\,\yh_{m-1}^\top\nonumber\\
    & = \W_m - \eta\, \bm{J}_{f_m}\!(\W^{\new}_m\yh_{m-1})^\top\big(\W^{\new}_m\yh_{m-1} - \a_m\big)\,\yh_{m-1}^\top\, .
\end{align}
Interestingly, the first iteration again recovers BackProp with an effective learning rate of $\eta_e = \eta\, \gamma$. We refer to this LocoProp-M variant with dual MD targets as \emph{PocoProp-M}.

\subsection{Other Possible Variants Based on the Regularizer Term}
So far, we discussed LocoProp (and PocoProp) variants by considering different local loss functions as well as targets. A  tangential approach to such constructions is to consider other possible options for the regularizer term on the weights. A number of possible alternatives are $\mathrm{L}_1$-regularizer, i.e., $\Vert\Wt - \W_m\Vert_1$ (for sparse updates), and a local Mahalanobis distance based on the Fisher matrix, that is
\[
\sfrac{1}{2}\, (\widetilde{\w} - \w_m)^\top \bm{F}(\w_m) (\widetilde{\w} - \w_m)\, ,
\]
where $\widetilde{\w}$ and $\w_m$ are the vectorized forms of $\Wt$ and $\W_m$, respectively, and $\bm{F}(\w_m)$ is the Fisher information matrix, defined in Eq.~\eqref{eq:fisher}. Note that a single iteration of LocoProp and PocoProp variants with this local Mahalanobis distance corresponds to a natural gradient descent step on $\W_m$ (see Eq.~\eqref{eq:ngd}). Thus, computationally effective approximations of such variants may yield an improved convergence. We also defer the analysis of such variants to future work.

\section{TUNING PROTOCOL AND SEARCH SPACES}
We setup a search space for each experiment as shown in Table~\ref{tbl:firstorder}, \ref{tbl:locoprop}, and \ref{tbl:secondorder}. We use a Bayesian optimization package for hyper-parameter search to optimize for training loss. We use decoupled weight decay of $10^{-5}$ across all experiments. We run up to ~4096 trails for all experiments, except for the wider/deeper sized autoencoder where we run only up to 512 trials. Experiments are conducted on V100 GPUs. Our rough estimate for the total GPU time is $\sim$15k GPU hours.
\begin{table*}
\begin{minipage}{0.45\textwidth}
\centering
\setlength{\extrarowheight}{3.5pt}
\resizebox{0.8\textwidth}{!}{
\begin{tabular}{ccc}
\toprule
Hyper-parameter & Range & Scaling \\ \midrule
$\eta$ & $\srange{e-7}{0.1}$ & Log \\ \hline
$1 - \beta_1$ & $\srange{e-3}{0.1}$ & Log \\ \hline
$1 - \beta_2$ & $\srange{e-3}{0.1}$ & Log \\ \hline
$\epsilon$ & $\srange{e-10}{e-5}$ & Log \\ \bottomrule
\end{tabular}
}
\vspace{0.1cm}
\caption{Search space for comparing first-order optimizers. Note that the hyper-parameters are dropped when not applicable (e.g., $\beta_2$ for non-adaptive methods).}
\label{tbl:firstorder}
\end{minipage}\,\,\,\,\,
\begin{minipage}{0.45\textwidth}
\vspace{-0.76cm}
\centering
\setlength{\extrarowheight}{3.5pt}
\resizebox{0.8\textwidth}{!}{
\begin{tabular}{ccc}
\toprule
Hyper-parameter & Range & Scaling \\ \midrule
$\eta$ & $\srange{e-7}{10}$ & Log \\ \hline
$1 - \beta_1$ & $\srange{e-3}{0.1}$ & Log \\ \hline
$1 - \beta_2$ & $\srange{e-3}{0.1}$ & Log \\ \hline
$\epsilon$ & $\srange{e-10}{e-1}$ & Log \\ \bottomrule
\end{tabular}
}
\vspace{0.1cm}
\caption{Search space for Shampoo and \mbox{K-FAC.}}
\label{tbl:secondorder}
\end{minipage}
\end{table*}

\begin{table}[h!]
\centering
\setlength{\extrarowheight}{3.5pt}
\resizebox{0.65\textwidth}{!}{
\begin{tabular}{ccc}
\toprule
Hyper-parameter & Range & Scaling \\ \midrule
$\gamma$ & [$10,1\times10^2,5\times10^2,10^3,2\times10^3,3\times10^3,5\times10^3$] & Log \\ \hline
$\eta$ & $\srange{e-7}{0.1}$ & Log \\ \hline
$1 - \beta_1$ & $\srange{e-3}{0.1}$ & Log \\ \hline
$1 - \beta_2$ & $\srange{e-3}{0.1}$ & Log \\ \hline
$\epsilon$ & $\srange{e-10}{e-5}$ & Log \\ \bottomrule
\end{tabular}
}
\vspace{0.2cm}
\caption{Search space for LocoProp (both variants). Base optimizer is RMSProp.}
\label{tbl:locoprop}
\end{table}



\begin{figure}[h!]
\vspace{-0.3cm}
\begin{center}
    \subfigure[MNIST with ReLU transfer function.]{\includegraphics[height=0.28\linewidth]{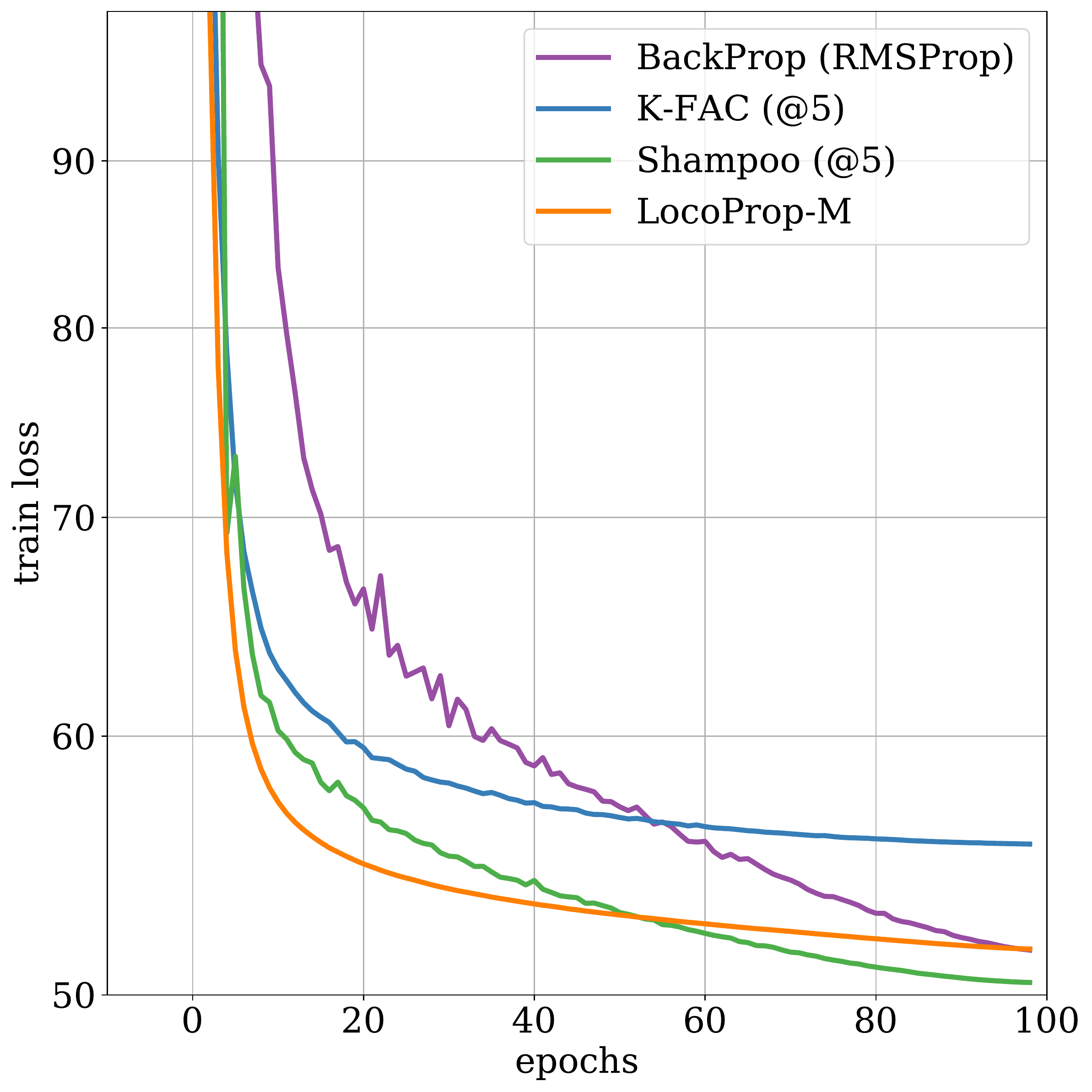}}  
    \hspace{0.1cm}    
    \subfigure[Fashion MNIST with $\tanh$ transfer function.]{\includegraphics[height=0.28\linewidth]{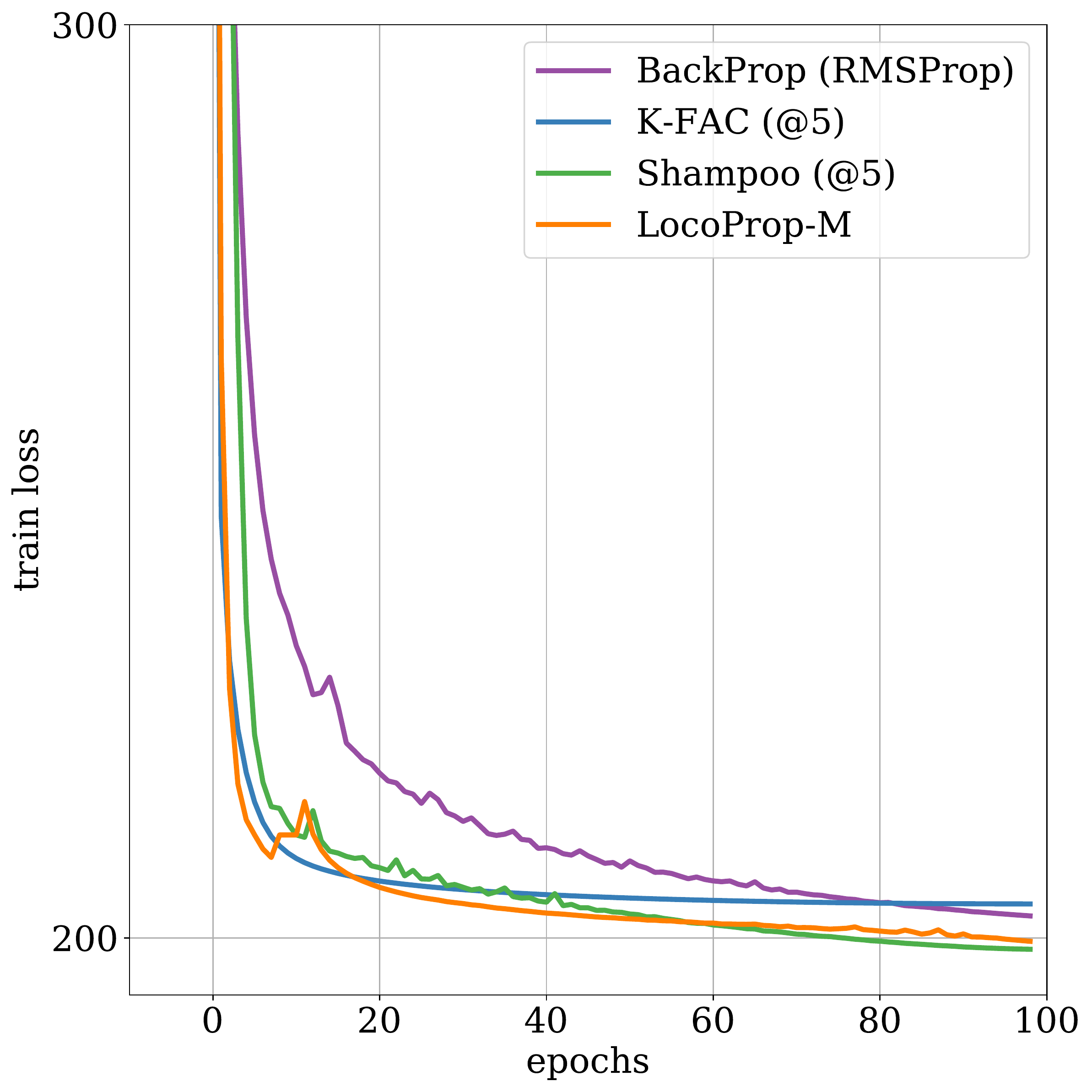}}
    \hspace{0.1cm}    
    \subfigure[CURVES with $\tanh$ transfer function.]{\includegraphics[height=0.28\linewidth]{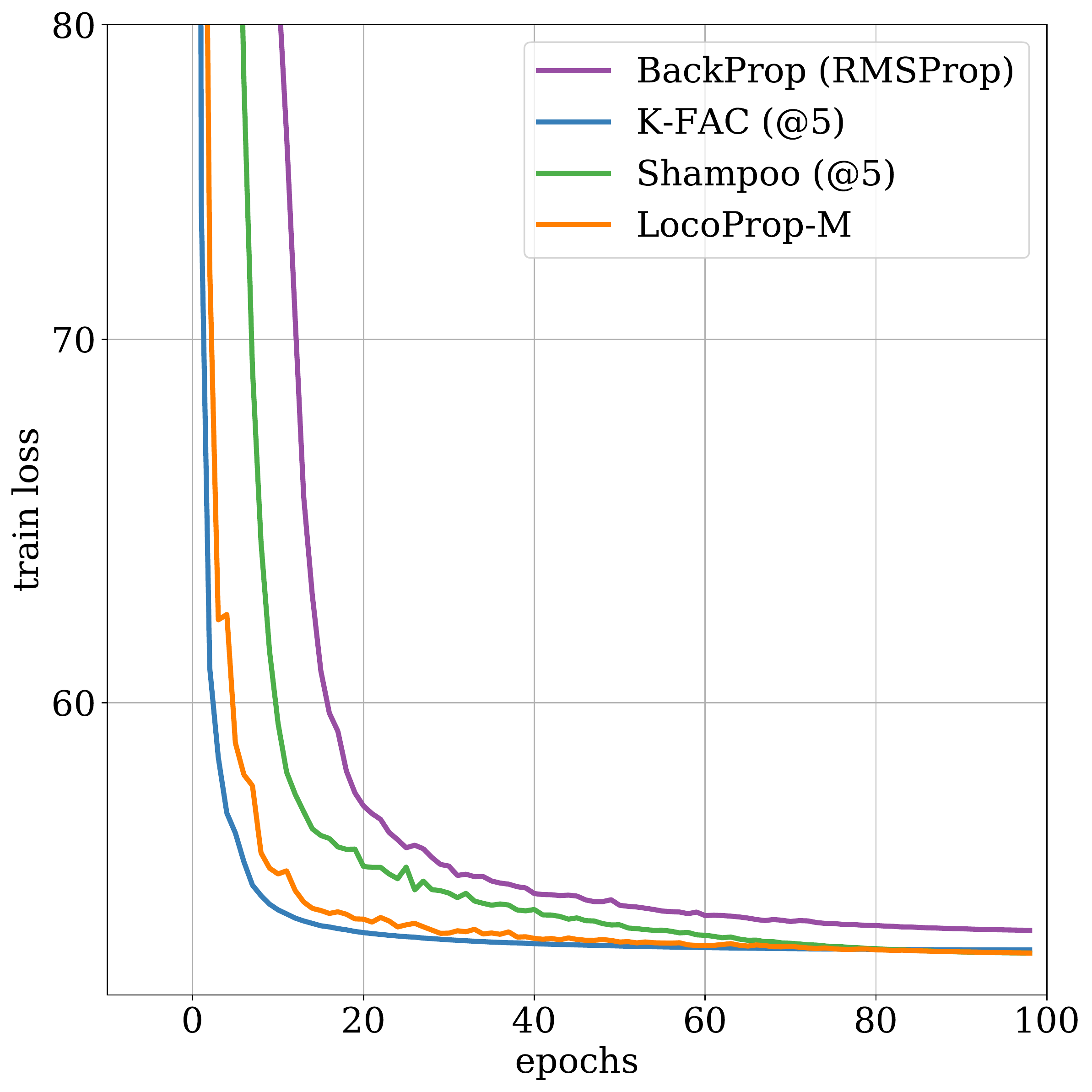}}\\[-0.1cm]
    \subfigure[Test loss on MNIST with ReLU transfer function.]{\includegraphics[height=0.28\linewidth]{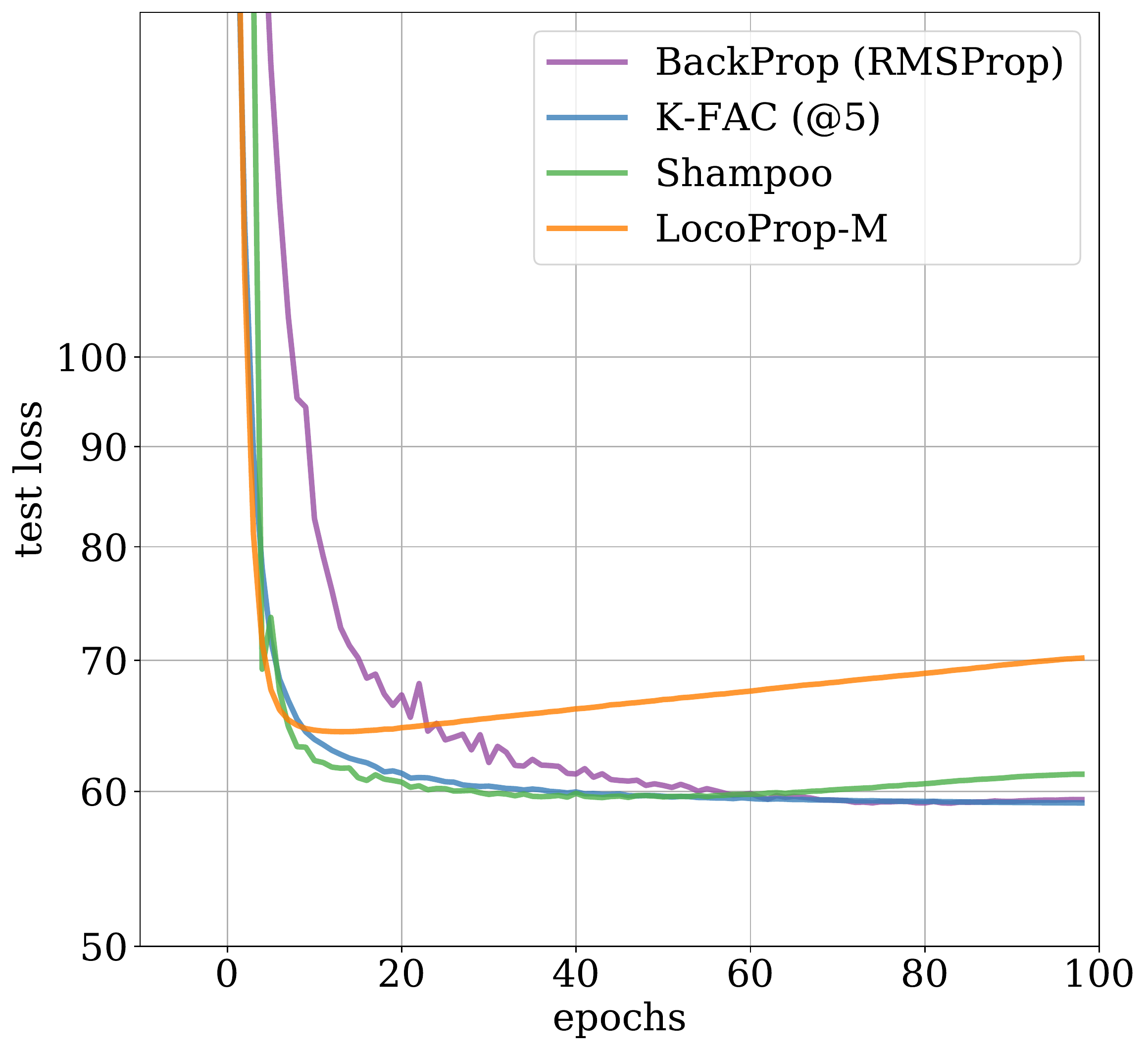}}    
    \hspace{0.1cm}  
    \subfigure[Test loss on Fashion MNIST with $\tanh$ transfer function.]{\includegraphics[height=0.28\linewidth]{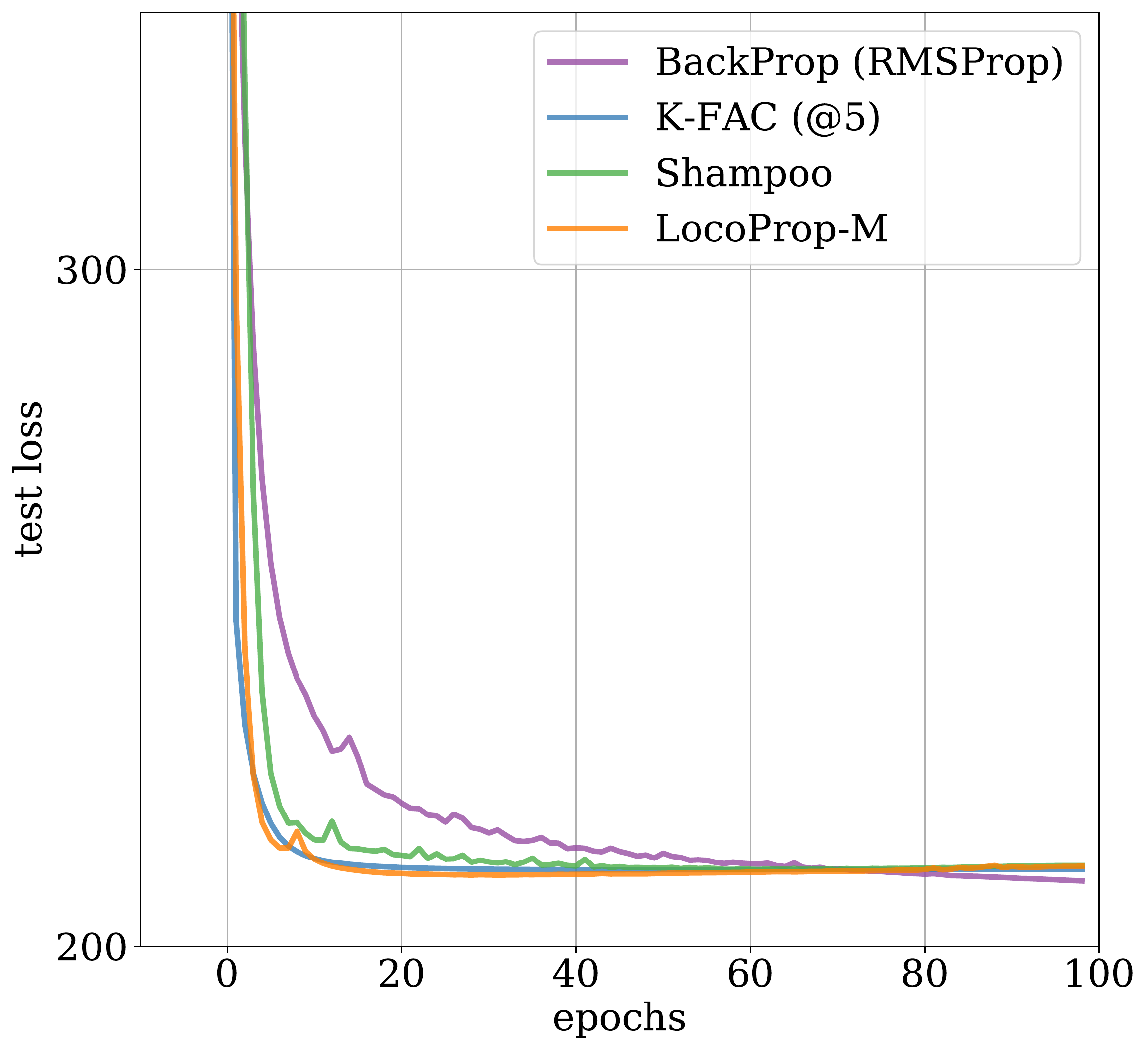}}  
    \hspace{0.1cm}    
    \subfigure[Test loss on CURVES with $\tanh$ transfer function.]{\includegraphics[height=0.28\linewidth]{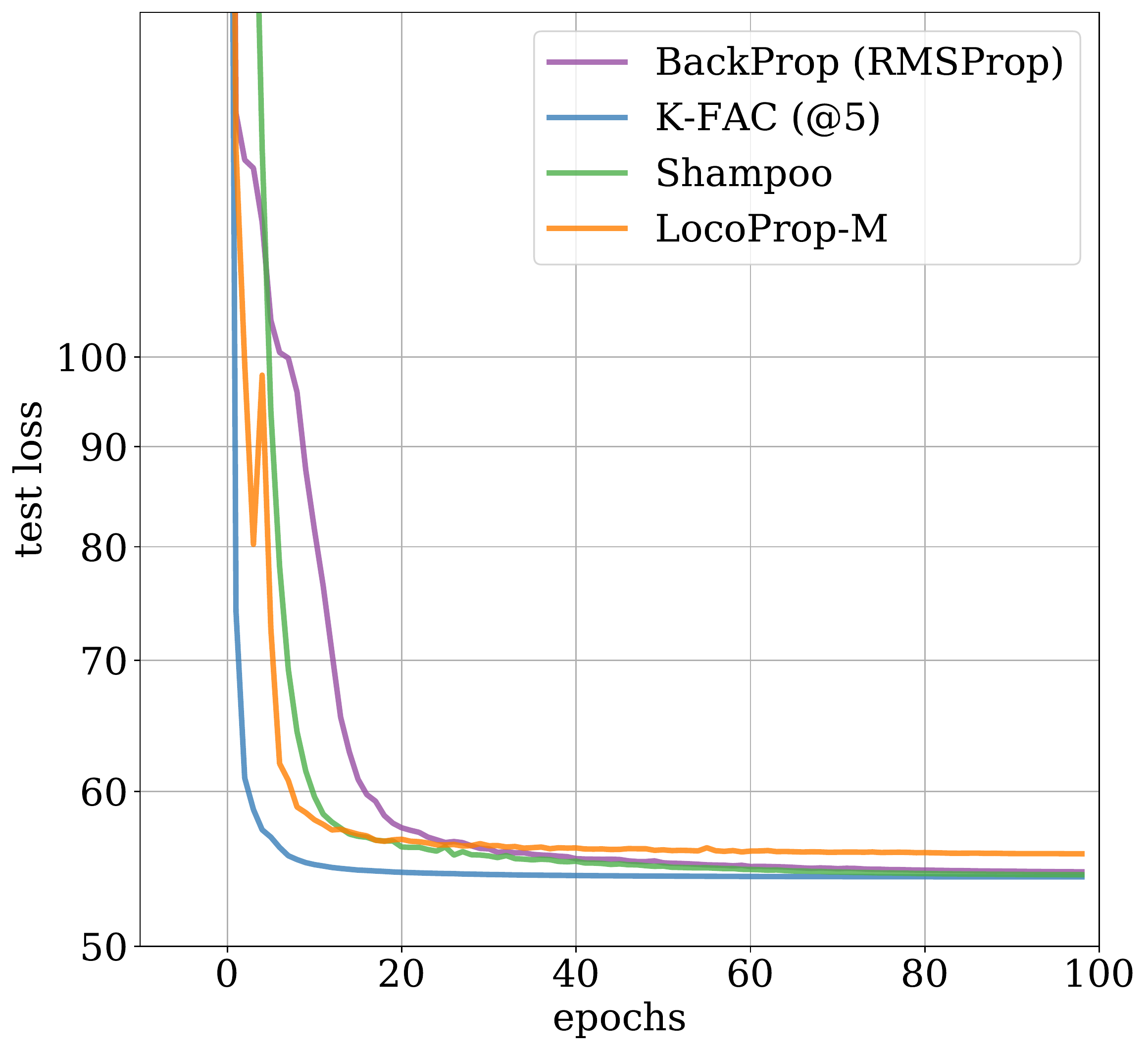}}  
    \vspace{-0.25cm}
    \caption{Results on different transfer functions and datasets. Comparisons on the standard autoencoder: (a) MNIST dataset and ReLU transfer function, (b) Fashion MNIST and (c) CURVES datasets using $\tanh$ transfer function. @k indicates the interval for carrying out inverse ($p$th root) operation.}
    \label{fig:relu_fmnist_curves}
    \end{center}
\end{figure}
\section{RESULTS ON RELU AND OTHER DATASETS}
We provide additional results using ReLU transfer function on the MNIST dataset in Figure~\ref{fig:relu_fmnist_curves}(a). Interestingly, K-FAC performs poorly on this problem and converges to a worse solution than RMSProp. LocoProp-M also performs well initially, but converges to a similar solution to RMSProp. Shampoo outperforms all other methods on this problem.

We also show results on the Fashion MNIST and CURVES datasets (using the standard autoencoder with $\tanh$ transfer function) in Figure~\ref{fig:relu_fmnist_curves}(b) and~\ref{fig:relu_fmnist_curves}(c), respectively. Again, K-FAC performs poorly on Fashion MNIST, but works well on the CURVES dataset. LocoProp-M performs closely to the best performing second-order method.

Lastly, we also show the test loss in Figure~\ref{fig:relu_fmnist_curves}(d)-(f) for these problems. The test loss results show signs of overfitting in all cases. The main reason for overfitting is that, in this work, we only focus on optimizing the fixed loss and thus, do not tune the weight decay regularizer term, which is set to ${1}\mathrm{e}{-5}$ in the default autoencoder definition. Similar behavior on this problem has been observed in previous work using the default hyper-parameters~\citep{goldfarb2020practical}.

\begin{figure}[h!]
\begin{center}
    \subfigure[\mbox{LocoProp}-S and \mbox{LocoProp}-M has comparable performance to second-order methods even with a smaller batch size.]{\includegraphics[height=0.28\linewidth]{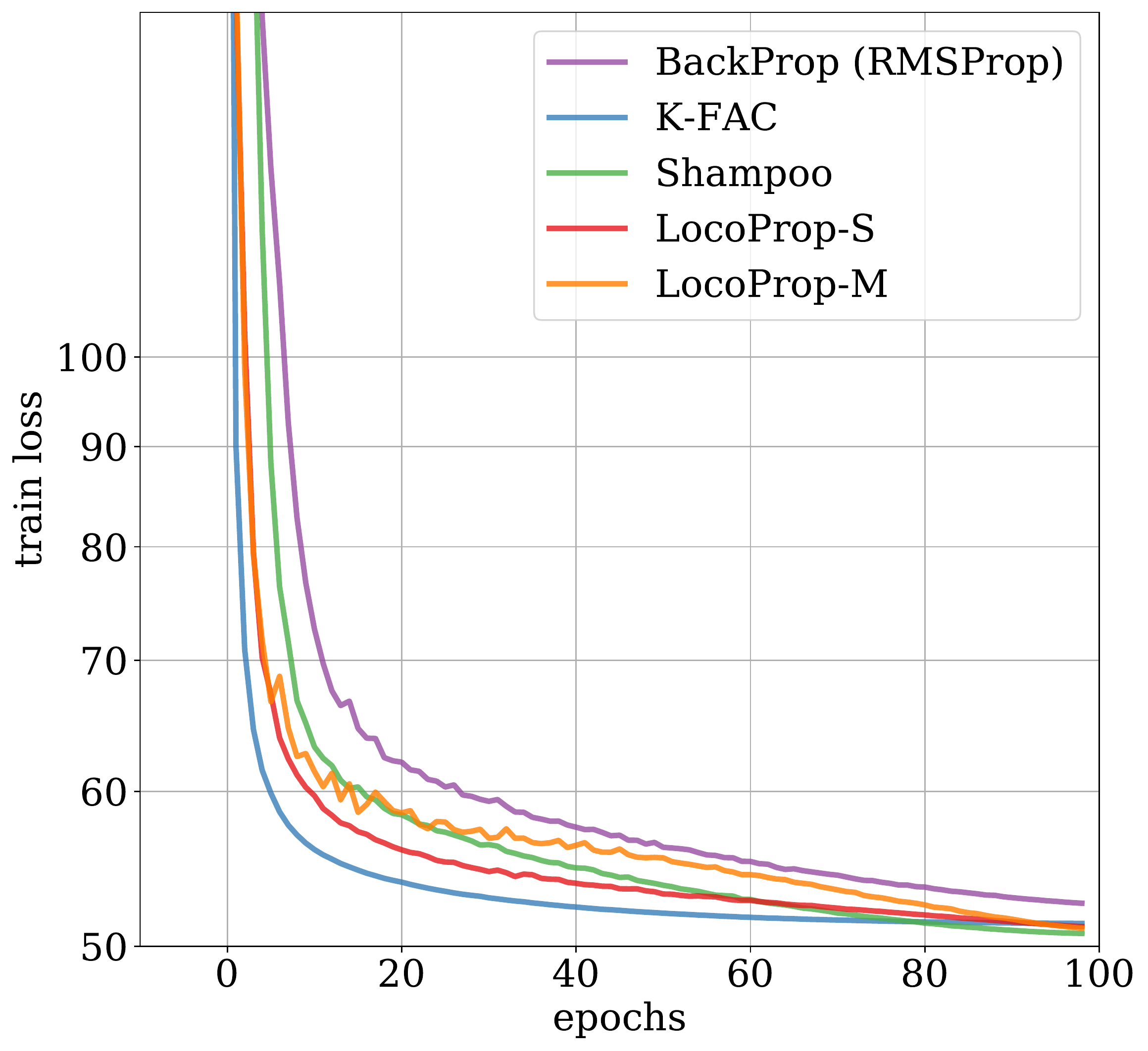}}
    \hspace{0.1cm}    
    \subfigure[Test loss on the same problem for different methods.]{\includegraphics[height=0.28\linewidth]{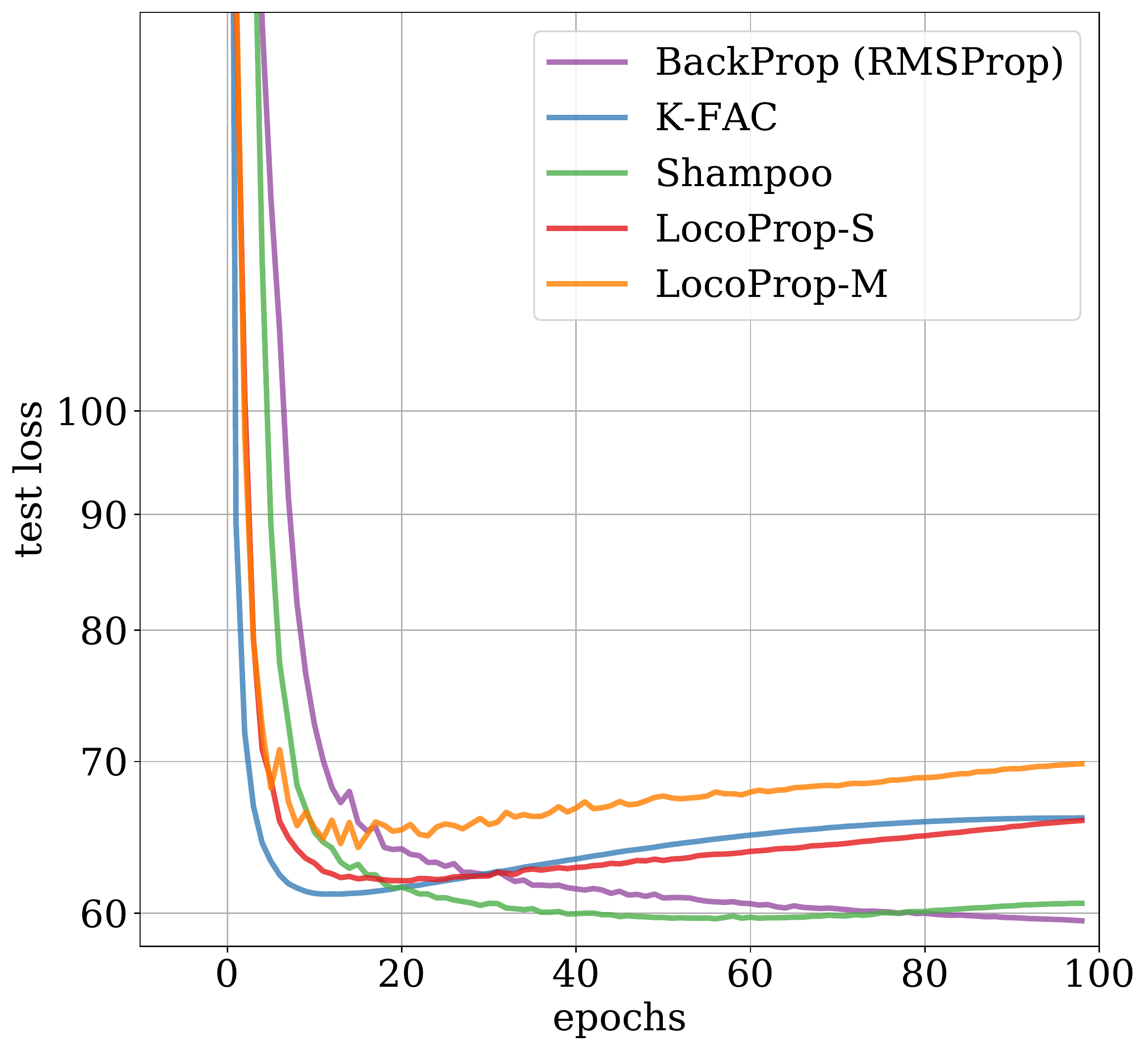}}
    \vspace{-0.25cm}
    \caption{MNIST autoencoder with $\tanh$ transfer function trained with batch size of 100.}
    \label{fig:mnist_batch_size}
    \end{center}
\end{figure}

\begin{figure}[h!]
\vspace{-0.3cm}
\begin{center}
    \subfigure[Test loss on the standard autoencoder with $\tanh$ transfer function.]{\includegraphics[height=0.28\linewidth]{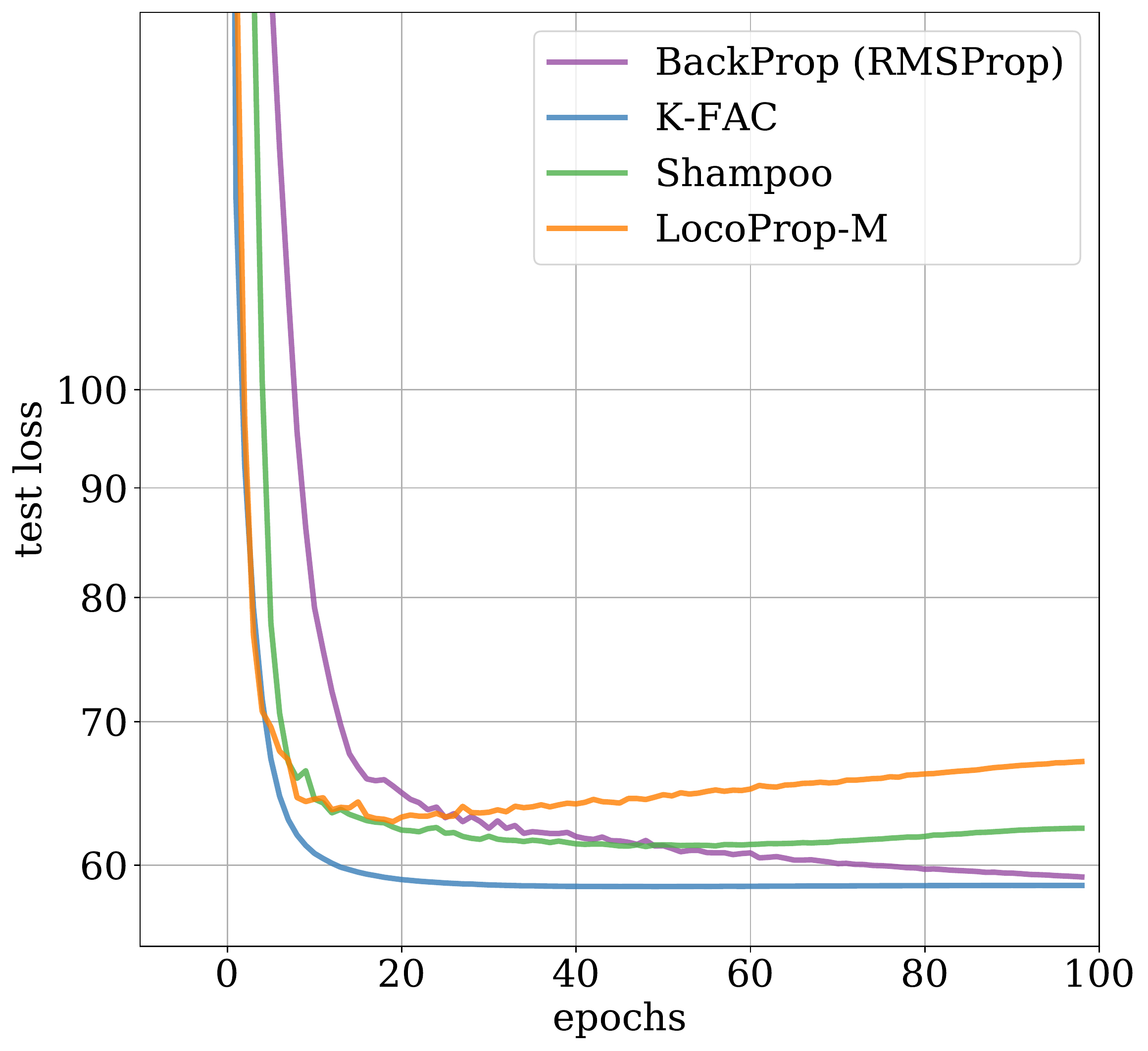}}
    \hspace{0.1cm} 
    \subfigure[Test loss on the wide autoencoder with $\tanh$ transfer function.]{\includegraphics[height=0.28\linewidth]{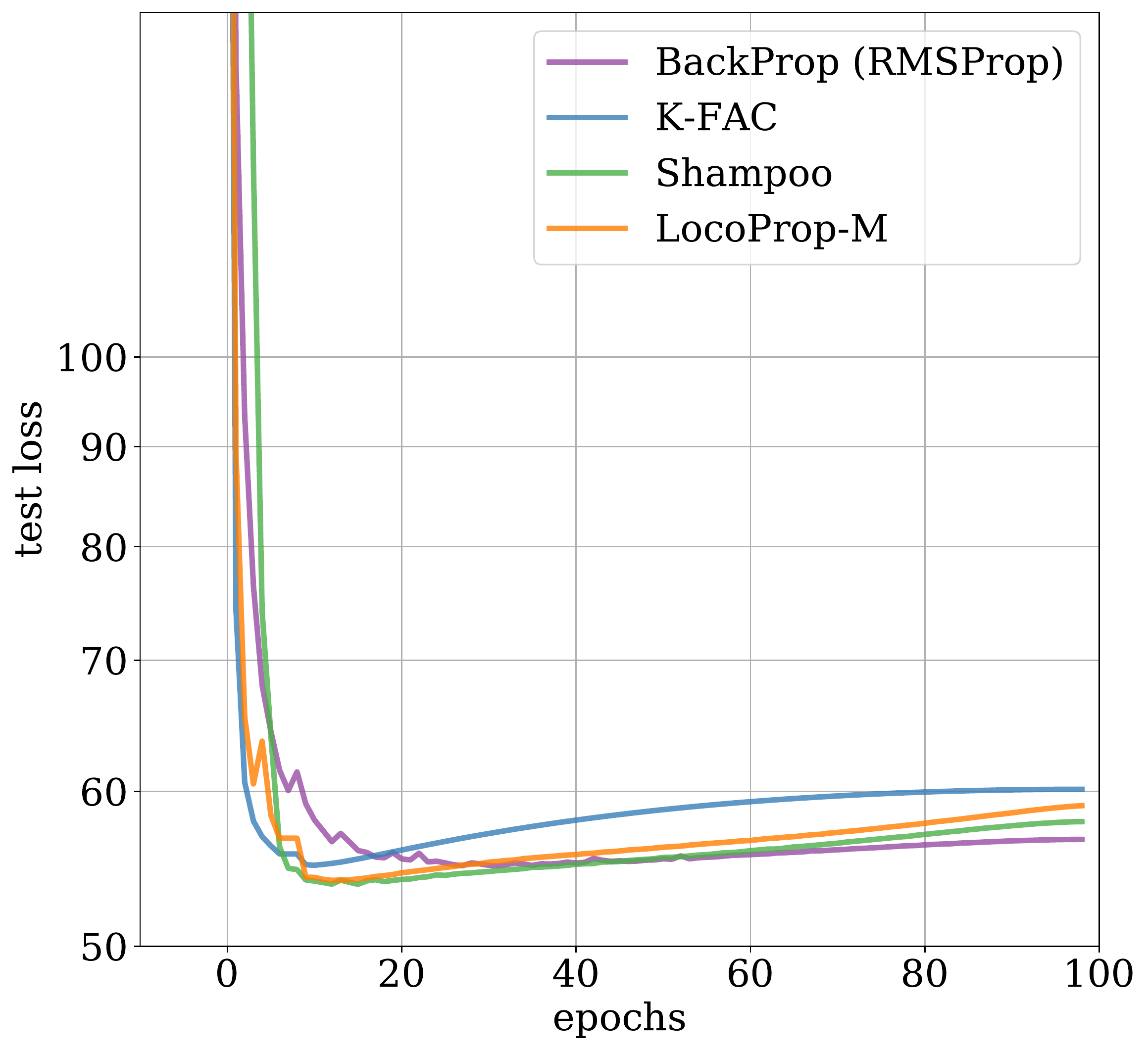}}  
    \hspace{0.1cm}    
    \subfigure[Test loss on the deep autoencoder with $\tanh$ transfer function.]{\includegraphics[height=0.28\linewidth]{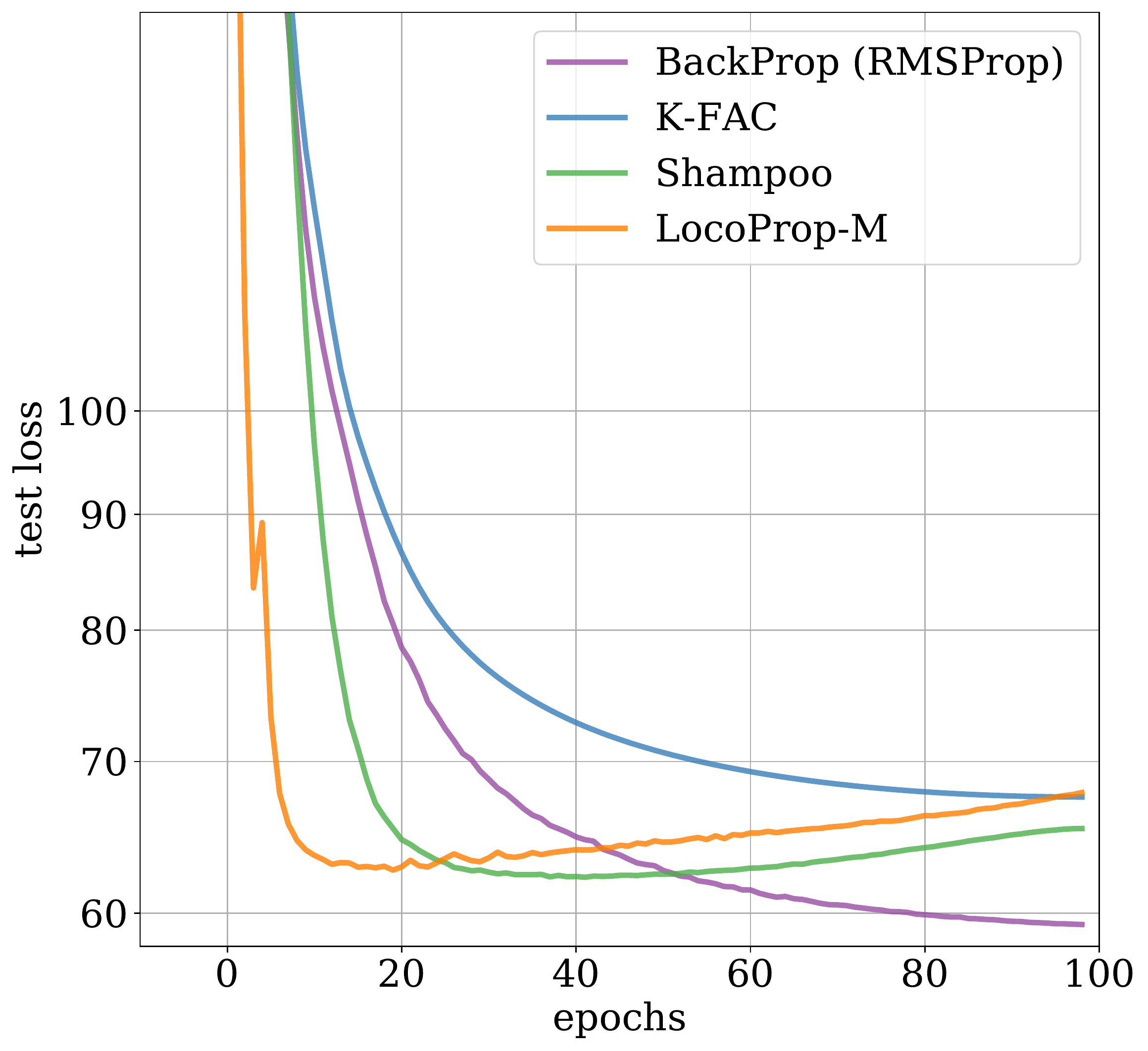}}
    \vspace{-0.25cm}
    \caption{Test loss results on the MNIST dataset with $\tanh$ \mbox{transfer} function. Comparisons on: (a) standard, (b) wide, and (c) deep autoencoder variants.}
    \label{fig:mnist_test_tanh}
    \end{center}
\end{figure}

\section{RESULTS WITH A SMALLER BATCH SIZE}
We train the standard autoencoder on the MNIST dataset with a batch size of 100. Results are presented in Figure~\ref{fig:mnist_batch_size}, which look similar to the earlier results at batch size 1000. One thing that stands out is that LocoProp-S works just as well in this setting as LocoProp-M.

\section{TEST LOSS FOR STANDARD, WIDE, AND DEEP VARIANTS}
We provide the plots of the test loss for all variants of the autoencoder model in Figure~\ref{fig:mnist_test_tanh}. In general, the autoencoder model does not generalize well when trained with second-order methods or LocoProp. As discussed earlier, this behavior is due to an undertuned weight decay regularizer parameter for each case. In this work, we only focus on minimizing the training loss and defer the analysis of the generalization properties of LocoProp to future work.

\section{HYPER-PARAMETERS FOR EXPERIMENTS}
The entire hyper-parameter sweeps along with the sensitivity analysis, training and test losses, and the code to reproduce the results is available at \url{https://github.com/google-research/google-research/tree/master/locoprop}.
\end{document}